\def\year{2022}\relax
\documentclass[letterpaper, final]{article}
\usepackage{aaai22}
\usepackage{natbib}
\usepackage{times}
\usepackage{helvet}
\usepackage{courier}
\usepackage{url}
\usepackage{graphicx}
\usepackage[compact]{titlesec}
\usepackage[T1]{fontenc}

\usepackage{multirow}
\usepackage[ruled,noline,noalgohanging]{algorithm2e}
\usepackage{stmaryrd}
\usepackage{tikz}
\usetikzlibrary{arrows,positioning,decorations.markings}
\tikzset{
  varnode/.style={rectangle,outer sep=0mm},
  varnodenoperi/.style={rectangle,outer sep=-1mm},
  ourarrow/.style={>=stealth}, 
  ourarc/.style={>=stealth,thick,arc}}
\usepackage{pgfplots}
\usepackage{amsthm}
\usepackage{amssymb}

\newtheorem{mylem}{Lemma}
\newtheorem{mydef}{Definition}

\newtheorem{mythm}{Theorem}

\usepackage{multicol}
\usepackage{latexsym}
\usepackage[inline]{enumitem}
\setlist[enumerate]{label=(\roman*)}
\usepackage[Euler]{upgreek}
\usepackage{tabularx}
\usepackage{mathtools}
\usepackage{mathtools}
\usepackage{subcaption}
\usepackage{isabelle, isabellesym}

\title{}
\author{\\}
\affiliations{}

\usepackage{textcomp}
\usepackage{listings}
\lstdefinestyle{inline}{
    basicstyle=\ttfamily\small
}

\usepackage[obeyFinal]{todonotes}
\presetkeys
    {todonotes}
    {inline,backgroundcolor=yellow,bordercolor=black,linecolor=cyan}{}
\tikzstyle{inlinenotestyle} = [notestyle,text width=\linewidth,inner sep=1em,outer sep=2pt,align=justify]
\usepackage{bm}

\usepackage[draft]{hyperref}
\usepackage{graphicx}

\usepackage[normalem]{ulem}
\usepackage{lineno}

\usepackage[most]{tcolorbox}
\tcbuselibrary{breakable,hooks,skins,theorems}
\newtcbtheorem{IsabelleTheorem}{Listing}{
    enhanced jigsaw
    , breakable
    , float=h
    , top=0pt
    , right=0pt
    , bottom=0pt
    , left=0pt
    , boxrule=0pt
    , toprule=1pt
    , bottomrule=1pt
    , titlerule=1pt
    , colframe=black
    , sharp corners
    , colbacktitle=white
    , coltitle=black
    , fonttitle=\tiny
    , colback=white
    , fontupper=\ttfamily\tiny
    , before upper={\parindent0em}
}{isa}

\lstdefinelanguage{isabelle}{
    morekeywords={record,type_synonym,definition,fun,function,primrec,where,lemma,theorem,unfolding,by,shows,assumes,and,datatype,using,abbreviation
,moreover,have,hence,thus,qed,proof,let,ultimately,show,next,in}
    , sensitive=true
    , showstringspaces=false
    , framerule=0pt
    , xleftmargin=2em
    , numbers=left
    , numberstyle=\ttfamily\tiny
    , firstnumber=1
    , stepnumber=2
    , basicstyle=\ttfamily\tiny
    , breaklines=true
    , showspaces=false
    , morecomment=[l]{--}
    , morecomment=[s]{(*}{*)}
    , commentstyle=\color{gray}
    , morestring=[b]"
    , literate={\\<times>}{{$\times$}}{1} {\\<equiv>}{{$\equiv$}}{1} {\\<forall>}{{$\forall$}}{1} {\\<exists>}{{$\exists$}}{1} {\\<and>}{{$\land$}}{1}
        {\\<in>}{{$\in$}}{1} {\\<Rightarrow>}{{$\Rightarrow$}}{1} {\\<lambda>}{{$\lambda$}}{1} {::}{{$::$}}{1}
        {\\<subseteq>}{{$\subseteq$}}{1} {\\<^sub>m}{{$_m$}}{1} {\\<longleftrightarrow>}{{$\longleftrightarrow$}}{3}
        {\\<pi>}{{$\pi$}}{1} {\\<delta>}{{$\delta$}}{1} {\\<lbrakk>}{{$\llbracket$}}{1} {\\<rbrakk>}{{$\rrbracket$}}{1}
        {\\<Longrightarrow>}{{$\Longrightarrow$}}{3} {\\<not>}{{$\lnot$}}{1} {\\<le>}{{$\le$}}{1} {\\<rightharpoonup>}{{$\rightharpoonup$}}{2}
        {\\<^sub>\\<V>}{{$_{\mathcal V}$}}{1} {\\<lparr>}{{$\llparenthesis$}}{1} {\\<rparr>}{{$\rrparenthesis$}}{1}
        {\\<leftarrow>}{{$\leftarrow$}}{1} {\\<^sub>\\<O>}{{$_{\mathcal O}$}}{1} {\\<^sub>I}{{$_{\texttt{I}}$}}{1}
        {\\<^sub>G}{{$_{\texttt{G}}$}}{2} {\\<phi>}{{$\varphi$}}{1} {\\<Phi>}{{$\Phi$}}{1} {\\<psi>}{{$\psi$}}{1} {\\<Psi>}{{$\Psi$}}{1}
        {\\<^sub>S}{{$_{\texttt S}$}}{1} {\\<inverse>}{{$^{-1}$}}{1} {\\<^sub>O}{{$_{\texttt O}$}}{1} {\\<^bold>\\<And>}{{$\bm\bigwedge$}}{1}
        {\\<^bold>\\<or>}{{$\bm\lor$}}{1} {\\<^sub>G}{{$_{\texttt G}$}}{1} {\\<Pi>}{{$\Pi$}}{1} {\\<^sub>I}{{$_{\texttt I}$}}{1} {\\<noteq>}{{$\neq$}}{1}
        {\\<bottom>}{{$\bot$}}{1} {\\<^sub>+}{{$_\texttt +$}}{1} {\\<^bold>\\<and>}{{$\bm\land$}}{1} {\\<^bold>\\<not>}{{$\bm\lnot$}}{1}
        {\\<^sub>1}{{$_1$}}{1} {\\<^sub>2}{{$_2$}}{1} {\\<A>}{{$\mathcal A$}}{1} {\\<Turnstile>}{{$\models$}}{2} {\\<^sub>\\<forall>}{{$_\forall$}}{1}
        {\\<^sub>0}{{$_0$}}{1} {\\<tau>}{{$\tau$}}{1}  {\\<^sub>\\<Omega>}{{$_\Omega$}}{1} {\\<^sub>V}{{$_V$}}{1} {\\<^bold>\\<Or>}{{$\bm\bigvee$}}{1}
        {\\<^sub>P}{{$_\texttt P$}}{1} {\\<^sub>X}{{$_\texttt X$}}{1} {\\<longrightarrow>}{{$\longrightarrow$}}{2} {\\<or>}{{$\lor$}}{1} {\\<^sub>\\<pi>}{{$_\pi$}}{1}
        {\\<^sub>s}{{$_s$}}{1} {\\<^sub>t}{{$_t$}}{1} {\\<^sub>a}{{$_a$}}{1} {\\<^sub>r}{{$_r$}}{1} {\\<^sub>t}{{$_t$}}{1} {\\<^sub>e}{{$_e$}}{1} {\\<^sub>n}{{$_n$}}{1} {\\<^sub>d}{{$_d$}}{1} {\\<^sub>i}{{$_i$}}{1} {\\<^sub>v}{{$_v$}}{1} {\\<^sub>j}{{$_j$}}{1} {\\<^sub>b}{{$_b$}}{1} {\\<inter>}{{$\cap$}}{1} {\\<union>}{{$\cup$}}{1} {\\<Union>}{{$\bigcup$}}{1} {\\<^sup>c\\<TTurnstile>\\<^sub>=}{{${}^c\models_=$}}{1}
        {\\<open>}{{<}}{1} {\\<close>}{{>}}{1} {\\<langle>}{{$\langle$}}{1} {\\<rangle>}{{$\rangle$}}{1} {\\<ge>}{{$\ge$}}{1}
}

\newtcblisting[auto counter]{IsabelleSnippet}[2][]{
    listing options={
        language=isabelle
    }
    , listing only
    , enhanced jigsaw
    , breakable
    , top=-.5em
    , right=0pt
    , bottom=-.5em
    , left=0pt
    , boxrule=0pt
    , toprule=1pt
    , bottomrule=1pt
    , titlerule=1pt
    , colframe=black
    , sharp corners
    , colbacktitle=white
    , coltitle=black
    , fonttitle=\tiny
    , colback=white
    , fontupper=\ttfamily\tiny
    , before upper={\parindent0em}
    , title=Listing \thetcbcounter: #2
    , #1
}

\lstdefinelanguage{pddl}{
  sensitive=false,    
  morecomment=[l]{;}, 
  alsoletter={:,-},   
  morekeywords={
    define,domain,problem,not,and,or,when,forall,exists,either,
    :domain,:requirements,:types,:objects,:constants,
    :predicates,:action,:parameters,:precondition,:effect,
    :fluents,:primary-effect,:side-effect,:init,:goal,
    :strips,:adl,:equality,:typing,:conditional-effects,
    :negative-preconditions,:disjunctive-preconditions,
    :existential-preconditions,:universal-preconditions,:quantified-preconditions,
    :functions,assign,increase,decrease,scale-up,scale-down,
    :metric,minimize,maximize,
    :durative-actions,:duration-inequalities,:continuous-effects,
    :durative-action,:duration,:condition
  },
  numberstyle=\ttfamily\tiny,
  basicstyle=\ttfamily\tiny
}

\newtcblisting[use counter from=IsabelleSnippet]{PddlListing}[2][]{
  listing options={
    language=pddl,
    escapechar=|
  }, 
  listing only, 
  enhanced jigsaw,
  breakable,
  top=-.5em,
  right=0pt,
  bottom=-.5em,
  left=0pt,
  boxrule=0pt,
  toprule=1pt,
  bottomrule=1pt,
  titlerule=1pt,
  colframe=black,
  sharp corners,
  colbacktitle=white,
  coltitle=black,
  fonttitle=\tiny,
  colback=white,
  fontupper=\ttfamily\tiny,
  before upper={\parindent0em},
  title=Listing \thetcbcounter: #2,
  #1
}

\lstdefinestyle{isainline}{
  language=isabelle,
  basicstyle=
    \ttfamily\small
}

\newcommand{\act}{\ensuremath{\pi}}

\makeatletter
\newcommand{\oset}[2]{
  {\mathop{#2}\limits^{\vbox to -1\ex@{\kern-\tw@\ex@
   \hbox{\scriptsize #1}\vss}}}}
\makeatother

\renewcommand{\vec}[1]{\ensuremath{\oset{$\rightarrow$}{\ensuremath{#1}}}}
\newcommand{\as}{\ensuremath{\vec{{\act}}}}

\newcommand{\satpreas}[2]{\ensuremath{sat_precond_as(s, \as)}}

\renewcommand{\v}{\ensuremath{\mathit{v}}}

\providecommand{\graph}{}
\providecommand{\cal}{}
\renewcommand{\cal}{}
\renewcommand{\graph}{{\cal G}}

\newcommand{\Omit}[1]{}

\newcommand{\mohammad}[1]{Mohammad: \textcolor{green}{#1}}
\newcommand{\lukas}[1]{Lukas: \textcolor{red}{#1}}

\newcommand{\snapshotsymbol}{|\kern-.7ex\raise.08ex\hbox{\scalebox{0.7}{$\bullet$}}}

\newcommand{\ssgraph}[1]{\graph_\ss}

\renewcommand{\prod}{\ensuremath{{{{{\mathlarger{\mathlarger {{\mathlarger {\Pi}}}}}}}}}}

\newcommand{\acta}{{\ensuremath{\act^1}}}
\newcommand{\actb}{{\ensuremath{\act^2}}}

\newcommand{\planningproblem}{\Uppi}

\tikzset{dots/.style args={#1per #2}{line cap=round,dash pattern=on 0 off #2/#1}}
\providecommand{\moham}[1]{\fbox{{\bf \@Mohammad: }#1}}

\newcommand{\goal}{{\ensuremath{\mathcal{G}}}}
\newcommand{\init}{{\ensuremath{\mathcal{I}}}}

\newcommand{\planningProb}{\planningproblem}
\newcommand{\propositions}{P}
\newcommand{\actions}{\delta}
\newcommand{\initState}{\init}
\newcommand{\goalState}{\goal}
\newcommand{\planningState}{M}
\newcommand{\snapAction}{\act}
\newcommand{\snapActionPre}{\snapAction_\textit{pre}}
\newcommand{\snapActionAdd}{\snapAction_\textit{add}}
\newcommand{\snapActionDel}{\snapAction_\textit{del}}
\newcommand{\snapActiona}{\acta}
\newcommand{\snapActionPrea}{{\snapActiona_\textit{pre}}}
\newcommand{\snapActionAdda}{{\snapActiona_\textit{add}}}
\newcommand{\snapActionDela}{{\snapActiona_\textit{del}}}
\newcommand{\snapActionb}{\actb}
\newcommand{\snapActionPreb}{{\snapActionb_\textit{pre}}}
\newcommand{\snapActionAddb}{{\snapActionb_\textit{add}}}
\newcommand{\snapActionDelb}{{\snapActionb_\textit{del}}}
\newcommand{\durAction}{\act}
\newcommand{\durActionStart}{\durAction_\textit{start}}
\newcommand{\durActionEnd}{\durAction_\textit{end}}
\newcommand{\durActionInv}{\durAction_\textit{inv}}

\newcommand{\plan}{\ensuremath{\as}}
\newcommand{\happening}{h}
\newcommand{\happeningActs}{A}
\newcommand{\happeningTime}{r}
\newcommand{\happTimePoint}{t}
\newcommand{\happTimePoints}{\textit{htps}}
\newcommand{\recSetActs}{B}
\newcommand{\recSetInvs}{I}

\newcommand{\floorType}{floor}
\newcommand{\elevatorType}{elevator}
\newcommand{\passengerType}{passenger}
\newcommand{\elevatorAtPred}{el-at}
\newcommand{\passengerAtPred}{p-at}
\newcommand{\inElevatorPred}{in-el}
\newcommand{\elevatorDoorOpenPred}{el-op}
\newcommand{\elevatorAtPredInline}[2]{\textit{\small(\elevatorAtPred~#1~#2)}}
\newcommand{\passengerAtPredInline}[2]{\textit{\small(\passengerAtPred~#1~#2)}}
\newcommand{\inElevatorPredInline}[2]{\textit{\small(\inElevatorPred~#1~#2)}}
\newcommand{\elevatorDoorOpenPredInline}[1]{\textit{\small(\elevatorDoorOpenPred~#1)}}
\newcommand{\elevatorDurationFunc}{el-dur}
\newcommand{\moveElevatorAct}{mv}
\newcommand{\openElevatorDoorAct}{op}
\newcommand{\closeElevatorDoorAct}{cl}
\newcommand{\enterElevatorAct}{en}
\newcommand{\exitElevatorAct}{ex}

\newcommand{\lstinlinemacro}[1]{\small\textit{#1}}
\newcommand{\insertActionAlgo}{\operatorname{insert-action}}
\newcommand{\simplifyActionAlgo}{\operatorname{simplify-action}}
\newcommand{\simplifyPlanAlgo}{\operatorname{simplify-plan}}
\newcommand{\validHapSeqAlgo}{\operatorname{valid-hap-seq}}
\newcommand{\checkPlanAlgo}{\operatorname{check-plan}}

\newcommand{\happSeq}{H}
\newcommand{\false}{\textit{False}}
\newcommand{\true}{\textit{True}}
 
\title{Formal Semantics and Formally Verified Validation for Temporal Planning}

\author{Mohammad Abdulaziz$^{1,2}$ {\normalfont and} Lukas Koller$^1$\\
{\normalfont \small $^{1}$Techniche Universität München, Germany}\\
{\normalfont \small $^{2}$King's College London, United Kingdom}
}
\begin{document}
\maketitle

\begin{abstract}
  We present a simple and concise semantics for temporal planning.
  Our semantics are developed and formalised in the logic of the interactive theorem prover Isabelle/HOL.
  We derive from those semantics a validation algorithm for temporal planning and show, using a formal proof in Isabelle/HOL, that this validation algorithm implements our semantics.
  We experimentally evaluate our verified validation algorithm and show that it is practical.

\end{abstract}

\renewcommand{\mohammad}[1]{}
\renewcommand{\lukas}[1]{}
\renewenvironment{isabelle}{
}{
  \ignorespacesafterend
}
\renewenvironment{isabellebody}{
}{
  \ignorespacesafterend
}
\renewenvironment{isamarkuptext}{
}{
  \ignorespacesafterend
}
\newcommand{\inputisabellesnippet}[1]{\noindent{\begin{minipage}[t]{0.46\textwidth}\renewcommand\normalsize{\tiny}
      \normalsize \texttt{\input{isabelle_snippets/output/document/#1}}\end{minipage}}}

\section{Introduction}
\label{sec:intro}

Although, performance-wise, planning algorithms and systems are very scalable and efficient, as shown by different planning competitions~\cite{ipc98,DBLP:journals/aim/ColesCOJLSY12,vallati20152014}, there is still to be desired when it comes to their trustworthiness, which is crucial to their wide adoption.
Consequently, there have been substantial efforts to improve the trustworthiness of planning systems~\cite{howey2004val,fox2005validating,eriksson2017unsolvability,abdulaziz2018formally,abdulaziz2018formallyVal,DBLP:conf/aaai/CimattiMR17,DBLP:conf/itp/AbdulazizGN19}.
A basic task when it comes to the trustworthiness of planning systems is that of \emph{plan validation}.
In its most basic form, this task is solved by a plan validator, which is a program that, given a planning problem and a candidate plan, confirms whether the candidate plan indeed solves the problem.
This boosts the trustworthiness of a plan chiefly because the plan validator should be a simple piece of software that can be more easily inspected than the planning system that computed the plan and, accordingly, less likely to have mistakes.

One challenge to plan validation is that the semantics of planning languages and formalisms can be too complicated.
This makes the validator a rather complicated piece of software defeating the trustworthiness appeal of the whole approach.
This is especially the case for advanced planning formalisms, like temporal planning~\cite{fox2003pddl2}, hybrid planning, and planning problems with processes and events~\cite{fox2002pddl+}.
This problem is further exacerbated by the low-level languages in which plan validators are usually implemented, e.g.\ the plan validation system used for most planning competitions, VAL~\cite{howey2004val}, is implemented in C++.
Another challenge to plan validation is that the semantics of planning languages have ambiguities, which lead to different interpretations of what constitutes a correct plan. 
E.g.\ there are multiple interpretations of sub-typing using ``Either'' in PDDL.

In this work we address the aforementioned challenges using an interactive theorem prover (ITP).
In particular, we use the ITP Isabelle/HOL~\cite{DBLP:books/sp/NipkowPW02}, which implements a formal mathematical system combining higher-order logic (HOL) and simple type theory.
Our first contribution is that we formally specify an abstract syntax for the temporal fragment of PDDL 2.1 in Isabelle/HOL and, based on that, formalise its semantics.
Compared to a pen-and-paper semantics, this has the advantage that it removes any room for ambiguity.
Furthermore, during formalising this fragment of PDDL, we found that certain parts of the semantics as specified by~\citeauthor{fox2003pddl2} could be simplified.
As our second contribution, we implement an executable plan validator for the temporal part of PDDL2.1 and we formally verify, using Isabelle/HOL, that it correctly implements the semantics which we formalised.
Our validator checks \begin{enumerate*}\item if a given problem and the candidate plan are well-formed, and \item if the candidate plan is indeed a solution to the problem.\end{enumerate*}
Lastly, we experimentally show that this validator is practical and compare it with VAL.
 
\newcommand{\egactionaDef}{(\openElevatorDoorAct\; e\ensuremath{_1})}
\newcommand{\egactionbDef}{(\enterElevatorAct\ p\ensuremath{_0}\ e\ensuremath{_1}\ f\ensuremath{_1})}
\newcommand{\egactioncDef}{(\enterElevatorAct\ p\ensuremath{_1}\ e\ensuremath{_0}\ f\ensuremath{_0})}
\newcommand{\egactiondDef}{(\closeElevatorDoorAct\ e\ensuremath{_0})}
\newcommand{\egactiona}{\textit{\small\egactionaDef}}
\newcommand{\egactionb}{\textit{\small\egactionbDef}}
\newcommand{\egactionc}{\textit{\small\egactioncDef}}
\newcommand{\egactiond}{\textit{\small\egactiondDef}}
\newcommand{\egactionaTiny}{\textit{\tiny\egactionaDef}}
\newcommand{\egactionbTiny}{\textit{\tiny\egactionbDef}}
\newcommand{\egactioncTiny}{\textit{\tiny\egactioncDef}}
\newcommand{\egactiondTiny}{\textit{\tiny\egactiondDef}}
\newcommand{\atoms}{\operatorname{atoms}}

\section{Background}
\label{sec:bg}

In this work we build upon previous work by~\citeauthor{abdulaziz2018formallyVal}.
In their work, they formalised the syntax and semantics of the STRIPS fragment of PDDL in Isabelle/HOL.
The syntax was based on a grammar by~\citeauthor{kovacscomplete}.
Their semantics have two parts: \begin{enumerate*} \item a part defining what it means for a PDDL domain, instance or plan to be well-formed and \item a part defining the execution semantics of PDDL.\end{enumerate*}
The most interesting aspect of well-formedness has to do with typing: since the grammar of PDDL allows for Either-supertype specifications of the form `obj - Either obj1 obj2$\cdots$', this leads to ambiguities in interpreting the sub-typing relation when, for instance, instantiating a parameter with an Either-type by an object of an Either-type.
In this situation, they took the interpretation that this is a valid substitution if each of the object types is reachable, in the sub-typing relation, from at least one of the parameter types.
For the execution semantics, they formalised execution semantics of grounded STRIPS in Isabelle/HOL and, based on that, specified the execution semantics of PDDL by instantiating PDDL action schemata into STRIPS ground actions.

Since most of our work here concerns action execution, which is defined at the level of ground actions, this entire paper discusses ground actions and grounded planning problems.
The main change we made at the lifted action/problem level to the formalisation by Abdulaziz and Lammich is that we add an action duration constraints as a syntactic element to the abstract syntax element modelling action schemata.
We skip here those (modified) definitions and assume that the ground problems and plans were obtained from well-formed PDDL problems and plans, e.g.\ all parameters to predicates and action schemata are well-typed and action durations in the plan respect the duration constraints in the action schemata.
Interested readers should consult the formalisation scripts.

\newcommand{\valuation}{\ensuremath{\mathcal{A}}}
\mohammad{We need to add equality to the logic. This means we need to have types.}
\newcommand{\truetruth}{1}
\newcommand{\falsetruth}{0}
\begin{mydef}[Propositional Formulae]
A propositional formula $\phi$ defined over a set of atoms $V$ is either \begin{enumerate*}\item the verum $\top$, \item an atom $v$, s.t.\ $v\in V$, \item a negated propositional formula $\neg \phi$, \item a conjunction of two propositional formulae $\phi_1\wedge \phi_2$, or \item a disjunction of propositional formulae $\phi_1\vee \phi_2$.
A valuation $\valuation$ 
is a mapping of $V$ to the set $\{\falsetruth,\truetruth\}$.\end{enumerate*}
A valuation $\valuation$ is a model for a formula $\phi$, written $\valuation\models\phi$, iff \begin{enumerate*} \item $\phi$ is the verum, \item if $\phi$ is an atom, then $\valuation(v)=\truetruth$, \item if $\phi$ a negated formula $\neg \phi$, then $\valuation\not\models\phi$, \item if $\phi$ is a conjunction $\phi_1\wedge \phi_2$, then $\valuation\models\phi_1$ and $\valuation\models\phi_2$, and \item if $\phi$ is a disjunction of propositional formulae $\phi_1\vee \phi_2$, then $\valuation\models\phi_1$ or $\valuation\models\phi_2$.
\end{enumerate*}
\end{mydef}

\noindent Note: sometimes, for notational economy, we treat a valuation as a set.
In such cases, a valuation $\valuation:V\rightarrow\{\falsetruth,\truetruth\}$ is interpreted as the set $\{\v\mid \valuation(\v) = \truetruth\}$ and a set of atoms $V$ is interpreted as a valuation which maps any $\v\in V$ to $\truetruth$, and everything else to $\falsetruth$.
Also, in the rest of this paper a \emph{state} is synonymous with a valuation.\footnote{In the formalisation by Abdulaziz and Lammich, on which we base our work, there is support for equalities. This is done by modelling states as sets of formulae. We omit these details here since they are orthogonoal to the the semantics of durative actions.}

\begin{mydef}[Planning Problem]
  \mohammad{What are the restrictions on the formulae in the preconditions?}
  A planning problem $\planningProb$ is a tuple $\langle \propositions,\actions,\initState,\goalState\rangle$, where
  \begin{enumerate*}
    \item $\propositions$ is a set of atoms, each of which is a state characterising proposition,
    \item $\actions$: set of actions, each of which is a tuple $\langle\durActionStart,\durActionEnd,\durActionInv\rangle$ where
  \begin{itemize*}
    \item $\durActionStart,\durActionEnd$ are start and end \emph{snap actions}, and
    \item $\durActionInv$ is a formula defined over the propositions $\propositions$.
  \end{itemize*}
   A snap action $\snapAction$ is a tuple $\langle\snapActionPre,\snapActionAdd,\snapActionDel\rangle$ where
  \begin{itemize*}
    \item $\snapActionPre$ is its precondition, a formula using propositions $P$, 
    \item $\snapActionAdd\subseteq \propositions$ are its positive effects, and
    \item $\snapActionDel\subseteq \propositions$ are its negative effects.
  \end{itemize*}
    \item $\initState$ is a valuation over $\propositions$, modelling the initial state, and
    \item $\goalState$ is the goal state condition, which is a propositional formula defined over $\propositions$.
  \end{enumerate*}
\end{mydef}
\mohammad{We should add to the appendix an explanation of equalities and world models.}
\mohammad{The following note is bogus: we have durative and snap actions. We should probably move it near the definition of a plan.}

\newcommand{\passenger}{\textit{p}}
\newcommand{\floor}{\textit{f}}
\newcommand{\elevator}{\textit{e}}
{\makeatletter
\def\old@comma{,}
\catcode`\,=13
\def,{\ifmmode\old@comma\discretionary{}{}{}\else\old@comma\fi}
\makeatother
\makeatletter
\def\old@dot{.}
\catcode`\.=13
\def.{\ifmmode\old@dot\discretionary{}{}{}\else\old@dot\fi}
\makeatother
 As a running example we use a planning problem, which models an elevator control situation.
There are two passengers (\passenger$_0$ and \passenger$_1$), who want to use two elevators (\elevator$_0$ and \elevator$_1$) to change floors (\floor$_0$ and \floor$_1$). 
The set of state characterising propositions for this planning problem is $\propositions\equiv\bigcup\{ \{\elevatorAtPredInline{e$_i$}{f$_j$},\passengerAtPredInline{p$_k$}{f$_j$},\inElevatorPredInline{p$_k$}{e$_i$},\elevatorDoorOpenPredInline{e$_i$}\} \mid 0 \leq i,j,k \leq 1\}$.
The propositions \elevatorAtPredInline{e$_i$}{f$_j$} and \passengerAtPredInline{p$_k$}{f$_j$} encode at which floor an elevator or a passenger currently is.
The proposition \inElevatorPredInline{p$_k$}{e$_i$} encodes whether a passenger is in an elevator or not.
The proposition \elevatorDoorOpenPredInline{e$_i$} encodes whether an elevator door is open. 
The initial state is 
$\initState\equiv\{\elevatorAtPredInline{e$_0$}{f$_0$},\elevatorAtPredInline{e$_1$}{f$_1$},
  \passengerAtPredInline{p$_0$}{f$_1$},\passengerAtPredInline{p$_1$}{f$_0$},
  \elevatorDoorOpenPredInline{e$_0$}\}$ and its goal is $\goalState\equiv\passengerAtPredInline{p$_0$}{f$_0$}\wedge\passengerAtPredInline{p$_1$}{f$_1$}$.
In the initial state passenger \passenger$_0$ is on floor \floor$_1$ and passenger \passenger$_1$ is on floor \floor$_0$.
Both passengers want to change floors: passenger \passenger$_0$ want to move to floor \floor$_0$ and passenger \passenger$_1$ wants to move to floor \floor$_1$.
This is specified in the goal state formula.
Among many actions, the problem has actions to open one elevator's door $\egactiona\equiv\langle\langle\neg\elevatorDoorOpenPredInline{e$_1$},\emptyset,\emptyset\rangle,\langle\top,\{\elevatorDoorOpenPredInline{e$_1$}\},\emptyset\rangle,\top\rangle$, to have each of the passengers enter one of the elevators $\egactionb\equiv\langle\langle\passengerAtPredInline{p$_0$}{f$_1$}\wedge\elevatorAtPredInline{e$_1$}{f$_1$},\emptyset,\emptyset\rangle, \langle\top,\{\inElevatorPredInline{p$_0$}{e$_1$}\},\{\passengerAtPredInline{p$_0$}{f$_1$}\}\rangle,\elevatorDoorOpenPredInline{e$_1$}\rangle$ and $\egactionc\equiv\langle\langle\passengerAtPredInline{p$_1$}{f$_0$}\wedge\elevatorAtPredInline{e$_0$}{f$_0$},\emptyset,\emptyset\rangle,\langle\top,\{\inElevatorPredInline{p$_1$}{e$_0$}\},\{\passengerAtPredInline{p$_1$}{f$_0$}\}\rangle,\elevatorDoorOpenPredInline{e$_0$}\rangle$, and to close an elevator's door $\egactiond\equiv\langle\langle\elevatorDoorOpenPredInline{e$_0$},\emptyset,\emptyset\rangle,\langle\top,\emptyset,\{\elevatorDoorOpenPredInline{e$_0$}\}\rangle,\top\rangle$.
Each one of the actions has the expected preconditions and effects; e.g.\ moving the elevator requires its door to be closed during the entire move action.
}
\begin{mydef}[Plan]
A plan is a sequence of tuples $\langle\durAction_0,t_0,d_0\rangle,\dots,\langle \durAction_n,t_n,d_n\rangle$, where, for $1\leq i\leq n$, $\durAction_i\in \actions$ is an action, $t_i\in \mathbb{Q}_{\geq 0}$ and $d_i\in \mathbb{Q}_{\geq 0}$ are rational numbers, to which we refer as the starting time point and the duration, respectively.
For a plan $\plan$, we call a sorted sequence $\happTimePoint_0,\dots,\happTimePoint_n$ of the set of rational numbers $\{t\mid\langle a,t,d\rangle\in\plan\} \cup \{t+d\mid\langle a,t,d\rangle\in\plan\}$ the happening time points of the plan, and we denote it by $\happTimePoints(\plan)$.
\end{mydef}
\lukas{We define a plan to be a sequence. Can we use the $\in$-symbol for a plan, e.g. plan action $\langle a,t,d\rangle\in\plan$?}

A valid plan for the elevator running example starts with the following four plan actions: $\langle\egactiona,0,1\rangle$, $\langle\egactionb,1.25,0.5\rangle$, $\langle\egactionc,2,1\rangle$, and $\langle\egactiond,3,1\rangle$.

A central question when it comes to the semantics of temporal planning is that of \emph{plan validity}.
A central notion for defining plan validity is that of action \emph{non-interference}.
\begin{mydef}[Non-interference]
  \label{def:act_non_interf}
  Snap actions $\snapActiona$ and $\snapActionb$ are non-interfering iff
  \begin{enumerate*}
    \item $\atoms(\snapActionPrea) \cap$ $ (\snapActionAddb\cup\snapActionDelb)=\emptyset$,
    \item $\atoms(\snapActionPreb)\cap(\snapActionAdda\cup\snapActionDela)$ $=\emptyset$,
    \item $\snapActionAdda\cap\snapActionDelb=\emptyset$, and
    \item $\snapActionAddb\cap\snapActionDela=\emptyset$.
  \end{enumerate*}
\end{mydef}

The first definition of PDDL~2.1 temporal plan validity was posed by~\citeauthor{fox2003pddl2}~\citeyear{fox2003pddl2}.
Here we outline their definitions informally, due to lack of space.
In their definitions, a central notion was that of a \emph{simple plan}, which can be thought of as a temporal plan whose actions all have zero duration.
Execution semantics of simple plans are similar to the semantics of $\forall$-step parallel plans~\cite{rintanen:06}: more than one action can execute at the same time, given that the actions are non-interfering.
A valid temporal plan is defined one that can be compiled into a valid simple plan.
In this compilation, each durative action $\act$ starting at a time point $t$ and which has duration $d$ is compiled to three snap actions with duration zero.
The first action is $\durActionStart$ and it is scheduled to execute at $t$ in the simple plan.
The second action is $\durActionEnd$ and it is scheduled to execute at $t+d$ in the simple plan.
The third is an action with precondition $\durActionInv$ and no effects, which is scheduled to execute in the simple plan multiple times.
It executes once between every two happening time points of the plan iff the two happening time points are between $t$ and $t+d$, inclusive.

\subsection{Isabelle/HOL}
An ITP is a program which implements a formal mathematical system, i.e.\ a formal language, in which definitions and theorem statements are written, and a set of axioms or derivation rules, using which proofs are constructed.
To prove a fact in an ITP, the user provides high-level steps of a proof, and the ITP fills in the details, at the level of axioms, culminating in a formal proof.

We performed the formalisation and the verification using the interactive theorem prover Isabelle/HOL~\cite{DBLP:books/sp/NipkowPW02}, which is a theorem prover for HOL.
Roughly speaking, HOL can be seen as a combination of functional programming with logic.
Isabelle/HOL supports the extraction of the functional fragment to actual code in various languages~\cite{haftmann2007code}.

Isabelle is designed for trustworthiness: following the Logic for Computable Functions approach (LCF)~\cite{milner1972logic}, a small kernel implements the inference rules of the logic, and, using encapsulation features of ML, it guarantees that all theorems are actually proved by this small kernel.
Around the kernel there is a large set of tools that implement proof tactics and high-level concepts like algebraic datatypes and recursive functions.
Bugs in these tools cannot lead to inconsistent theorems being proved, but only to error messages when the kernel refuses a proof.

All the definitions, theorems and proofs in this paper have been formalised in Isabelle/HOL.
The formalisation can be found online\footnote{DOI:10.5281/zenodo.5784579}.
Usually, some definitions are best represented formally in a way which is different from how they represented informally.
For instance, a for-loop or a function applied to an indexed sequence in the informal definition are formalised in Isabelle/HOL as recursions over lists.
However, there is always a clear resemblance between the formal and the informal definitions and we provide a description associated with the formal definitions.
 \section{Semantics of Temporal Planning}
\label{sec:}

One issue with \citeauthor{fox2003pddl2}'s definition of plan validity is that it is too close to an operational specification of a validation algorithm for temporal plans.
A negative consequence of that becomes evident when trying to formalise the semantics and pin down all the details: the definitions then become very complicated and unreadable.
Although the need for simplifying definitions is generally evident, that need is exacerbated when the definitions are used as specifications against which we formally verify a validator.
In that scenario, the semantics should also provide a description of what the validator should do and they should be easily understandable through visual inspection.
We resolve that by providing a description of the semantics that abstractly describes what a valid plan is, without appealing to algorithmic constructions like the one of induced happening sequences.
We then show that our new definitions are equivalent to the operational definitions of \citeauthor{fox2003pddl2}.

\begin{mydef}[Valid State Sequence]
  \label{def:rec_sets}
  For $t\in\mathbb{Q}_{\ge 0}$ and a plan $\as$, let $\recSetActs_\happTimePoint\equiv\{\durActionStart\mid\langle\durAction,\happTimePoint,d\rangle\in\plan\}\cup\{\durActionEnd\mid\langle\durAction,\happTimePoint-d,d\rangle\in\plan\}$ and $\recSetInvs_\happTimePoint\equiv\{\durActionInv\mid\langle\durAction,\happTimePoint',d\rangle\in\plan\wedge \happTimePoint' <\happTimePoint<\happTimePoint'+d\}$.
Also, let $\happTimePoint_0,\dots,\happTimePoint_n$ be the happening time points of $\as$.
  For a sequence of states $\planningState_0,\dots,\planningState_{n+1}$, we say the sequence of states is valid wrt a plan $\as$ iff, for every happening time point $\happTimePoint_i$ of $\plan$, we have:
  \begin{enumerate*}
    \item $\planningState_i\models\durActionInv$, for every $\durActionInv\in\recSetInvs_{\happTimePoint_i}$,
    \item $\planningState_i\models\snapActionPre$, for every $\snapAction\in\recSetActs_{\happTimePoint_i}$,
    \item $\recSetActs_{\happTimePoint_i}$ is pairwise non-interfering, and
    \item $\planningState_{i+1}=(\planningState_i-\bigcup_{\snapAction\in\recSetActs_{\happTimePoint_i}}\snapActionDel)\cup\bigcup_{\snapAction\in\recSetActs_{\happTimePoint_i}}\snapActionAdd$.
  \end{enumerate*}
\end{mydef}

\begin{mydef}[Valid Plan]
  \label{def:plan_validity_rec}
  Plan $\plan$ is a valid plan for a problem $\planningproblem$ iff there is a state sequence $\planningState_1,\dots,\planningState_{n+1}$ s.t.\ $\initState,\planningState_1,\dots,\planningState_{n+1}$ is valid wrt $\as$ and $\planningState_{n+1}\models\goalState$.
\end{mydef}

Note: above, simultaneous execution of instantaneous ground actions is only allowed for non-interfering ground actions.
Otherwise, simultaneous execution might result in a not well-defined state.
We also use the same ground action interference condition defined by~\citeauthor{fox2003pddl2}.
\mohammad{Nonetheless, our results are valid for other definitions of interference that guarantee that resulting states are unambiguous.}

\begin{figure}
  \centering
  \begin{tikzpicture}[label distance=0mm,scale=1.15]  
  \draw (0,1.4) -- (0,1.6);
  \draw (0,1.5) node[label={[yshift=-1mm]above:{\tiny $\mathstrut 0$}}] {} -- (3,1.5);
  \draw (3,1.4) -- (3,1.6);
  \draw (3,1.5) node[label={[yshift=-1mm]above:{\tiny $\mathstrut 0.75$}}] {} -- (4,1.5);
  \draw (4,1.4) -- (4,1.6);
  \draw (4,1.5) node[label={[yshift=-1mm]above:{\tiny $\mathstrut 1$}}] {} -- (5,1.5);
  \draw (5,1.4) -- (5,1.6);
  \draw (5,1.5) node[label={[yshift=-1mm]above:{\tiny $\mathstrut 1.25$}}] {} -- (6,1.5);
  \draw (6,1.4) -- (6,1.6);
  \draw (6,1.5) node[label={[yshift=-1mm]above:{\tiny $\mathstrut 1.5$}}] {} -- (6.5,1.5);
  \draw[dotted] (6.5,1.5) -- (7.0,1.5);

  \draw (5.5,1.45) -- (5.5,1.55);
  \draw (4.5,1.45) -- (4.5,1.55);
  \draw (3.5,1.45) -- (3.5,1.55);
  \draw (1.5,1.45) -- (1.5,1.55);

  \filldraw[fill=green!80!black!15!white,draw=green!80!black!75!white,line width=0.2mm] (0,0.85) rectangle (4,1.35) {};
  \fill[fill=yellow!60!red!15!white] (5,0.85) rectangle (6.5,1.35);
  \draw[draw=yellow!60!red!75!white,line width=0.2mm] (5,0.85) -- (6.5,0.85);
  \draw[draw=yellow!60!red!75!white,line width=0.2mm] (5,0.85) -- (5,1.35);
  \draw[draw=yellow!60!red!75!white,line width=0.2mm] (5,1.35) -- (6.5,1.35);
  \draw[dotted] (6.4,1.1) -- (6.75,1.1);
  \filldraw[fill=yellow!60!red!15!white,draw=yellow!60!red!75!white,line width=0.2mm] (3,0.2) rectangle (5,0.7);
  \fill[fill=green!80!black!15!white] (6,0.2) rectangle (6.5,0.7);
  \draw[draw=green!80!black!75!white,line width=0.2mm] (6,0.2) -- (6.5,0.2);
  \draw[draw=green!80!black!75!white,line width=0.2mm] (6,0.2) -- (6,0.7);
  \draw[draw=green!80!black!75!white,line width=0.2mm] (6,0.7) -- (6.5,0.7);
  \draw[dotted] (6.4,0.45) -- (6.75,0.45);

  \node at (2,1.1) {\tiny $\mathstrut \egactionaTiny$};
  \node at (5.75,1.1) {\tiny $\mathstrut \egactionbTiny$};
  \node at (4,0.45) {\tiny $\mathstrut \egactioncTiny$};
  \node at (6.25,0.45) {\tiny $\mathstrut \egactiondTiny$};

  \filldraw[fill=blue!10!white,draw=blue!65!white,line width=0.2mm] (6,-0.25) circle (0.125);
  \node[label={below:{\tiny $\mathstrut \recSetInvs_{1.5}$}}] at (6,-0.25) {};
  \filldraw[fill=blue!10!white,draw=blue!65!white,line width=0.2mm] (5,-0.25) circle (0.125);
  \node[label={below:{\tiny $\mathstrut \recSetInvs_{1.25}$}}] at (5,-0.25) {};
  \filldraw[fill=blue!10!white,draw=blue!65!white,line width=0.2mm] (4,-0.25) circle (0.125);
  \node[label={below:{\tiny $\mathstrut \recSetInvs_1$}}] at (4,-0.25) {};
  \filldraw[fill=blue!10!white,draw=blue!65!white,line width=0.2mm] (3,-0.25) circle (0.125);
  \node[label={below:{\tiny $\mathstrut \recSetInvs_{0.75}$}}] at (3,-0.25) {};
  \filldraw[fill=blue!10!white,draw=blue!65!white,line width=0.2mm] (0,-0.25) circle (0.125);
  \node[label={below:{\tiny $\mathstrut \recSetInvs_0$}}] at (0,-0.25) {};
  \draw[dotted] (6.4,-0.25) -- (6.75,-0.25);

  \filldraw[fill=blue!20!white,draw=blue!80!white,line width=0.2mm] (6,-0.9) circle (0.125);
  \node[label={below:{\tiny $\mathstrut \recSetActs_{1.5}$}}] at (6,-0.9) {};
  \filldraw[fill=blue!20!white,draw=blue!80!white,line width=0.2mm] (5,-0.9) circle (0.125);
  \node[label={below:{\tiny $\mathstrut \recSetActs_{1.25}$}}] at (5,-0.9) {};
  \filldraw[fill=blue!20!white,draw=blue!80!white,line width=0.2mm] (4,-0.9) circle (0.125);
  \node[label={below:{\tiny $\mathstrut \recSetActs_1$}}] at (4,-0.9) {};
  \filldraw[fill=blue!20!white,draw=blue!80!white,line width=0.2mm] (3,-0.9) circle (0.125);
  \node[label={below:{\tiny $\mathstrut \recSetActs_{0.75}$}}] at (3,-0.9) {};
  \filldraw[fill=blue!20!white,draw=blue!80!white,line width=0.2mm] (0,-0.9) circle (0.125);
  \node[label={below:{\tiny $\mathstrut \recSetActs_0$}}] at (0,-0.9) {};
  \draw[dotted] (6.4,-0.9) -- (6.75,-0.9);

  \draw[dotted] (6,1.4) -- (6,-0.9);
  \draw[dotted] (5,1.4) -- (5,-0.9);
  \draw[dotted] (4,1.4) -- (4,-0.9);
  \draw[dotted] (3,1.4) -- (3,-0.9);
  \draw[dotted] (0,1.4) -- (0,-0.9);

  \node[label={[xshift=-3mm]right:{\tiny $\mathstrut \recSetInvs_0=\{\}$}}] at (0.0,-1.6) {};
  \node[label={[xshift=-3mm]right:{\tiny $\mathstrut \recSetInvs_{0.75}=\{\egactionaTiny_\textit{inv}\}$}}] at (0.0,-1.9) {};
  \node[label={[xshift=-3mm]right:{\tiny $\mathstrut \recSetInvs_1=\{\egactionaTiny_\textit{inv},\egactioncTiny_\textit{inv}\}$}}] at (0.0,-2.2) {};
  \node[label={[xshift=-3mm]right:{\tiny $\mathstrut \recSetInvs_{1.25}=\{\egactioncTiny_\textit{inv}\}$}}] at (0.0,-2.5) {};
  \node[label={[xshift=-3mm]right:{\tiny $\mathstrut \recSetInvs_{1.5}=\{\egactionbTiny_\textit{inv}\}$}}] at (0.0,-2.8) {};

  \node[label={[xshift=-3mm]right:{\tiny $\mathstrut \recSetActs_0=\{\egactionaTiny_\textit{start}\}$}}] at (3.5,-1.6) {};
  \node[label={[xshift=-3mm]right:{\tiny $\mathstrut \recSetActs_{0.75}=\{\egactioncTiny_\textit{start}\}$}}] at (3.5,-1.9) {};
  \node[label={[xshift=-3mm]right:{\tiny $\mathstrut \recSetActs_1=\{\egactionaTiny_\textit{end}\}$}}] at (3.5,-2.2) {};
  \node[label={[xshift=-3mm]right:{\tiny $\mathstrut \recSetActs_{1.25}=\{\egactioncTiny_\textit{end},$}}] at (3.5,-2.5) {};
  \node[label={[xshift=-3mm]right:{\tiny $\mathstrut \egactioncTiny_\textit{start}\}$}}] at (4.7,-2.8) {};
  \node[label={[xshift=-3mm]right:{\tiny $\mathstrut \recSetActs_{1.5}=\{\egactiondTiny_\textit{start}\}$}}] at (3.5,-3.1) {};
\end{tikzpicture}
   \caption{Concepts from Def.~\ref{def:rec_sets} for the elevator example.
  }\label{fig:new_semantics_illustration}
\end{figure}
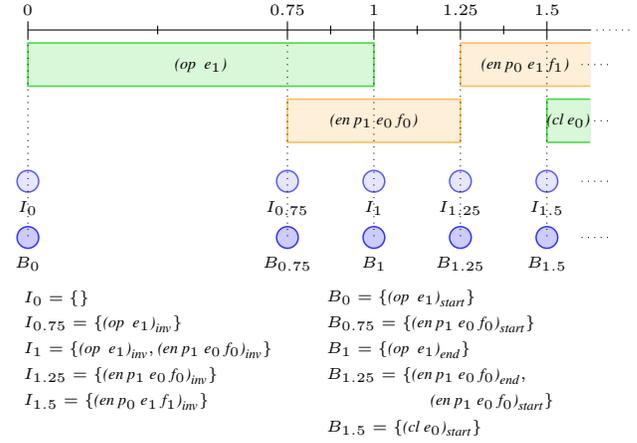

Figure~\ref{fig:new_semantics_illustration} illustrates the beginning of the instantiation of the elevator running example for Def.~\ref{def:rec_sets}. At the top of the illustration a timeline is depicted. Below the timeline the first four actions from the valid plan are shown. At the bottom of the illustration the individual sets needed for the state sequence are shown.

\subsection{Refining the Semantics Towards Executability}

A main goal of this paper is to construct a plan validator which is formally verified wrt the semantics.
We do that by following a step-wise refinement approach~\cite{DBLP:journals/cacm/Wirth71}, where we start from the abstractly specified semantics and refine that specification towards an executable program which fulfils those abstractly specified semantics.
The next step to refine our semantics is to obtain a version that is closer to the executable program.
In this version, we closely follow the semantics given by~\citeauthor{fox2003pddl2}.
A central concept in defining the semantics of temporal plans is that of \emph{happening sequences}.
Intuitively, these are the instantaneous changes that happen over the course of plan execution.

\begin{mydef}[Valid Happening Sequence]\label{def:val_happ_seq}
  A happening $\happening$ is a pair $\langle\happeningActs,\happeningTime\rangle$, where $\happeningActs$ is a set of snap actions and $\happeningTime\in\mathbb{Q}_{\geq 0}$ is the starting time point.
  For a happening sequence $\langle\happeningActs_0,\happeningTime_0\rangle,\dots,\langle\happeningActs_{n},\happeningTime_{n}\rangle$ and a state $\planningState_0$, we call a state sequence $\planningState_1,\dots,\planningState_{n+1}$ to be induced by $\planningState_0$ and the happening sequence iff for every $0 \leq i < m$
  \begin{enumerate*}
    \item $\planningState_i\models\snapActionPre$, for every $\snapAction\in\happeningActs_i$,
    \item $\happeningActs_i$ is pairwise non-interfering, and
    \item $\planningState_{i+1}=\left(\planningState_i-\bigcup_{\snapAction\in\happeningActs_i}\snapActionDel\right)\cup\bigcup_{\snapAction\in\happeningActs_i}\snapActionAdd$, for $0 \leq i \leq n$.
  \end{enumerate*}
  A happening sequence is valid wrt some state iff they induce a valid state sequence.
\end{mydef}

A happening sequence which models the effects and executability of a temporal plan is called an \emph{induced happening sequence}.
The validity of a temporal plan is defined as the validity of the induced happening sequence.

\begin{mydef}[Induced Happening Sequence]
\label{def:ind_happ_seq}
A happening sequence $\langle\happeningActs_0,\happeningTime_0\rangle,\dots,\langle\happeningActs_m,\happeningTime_m\rangle$ is an induced happening sequence for a plan $\plan$ with happening time points $\happTimePoint_0,\dots,\happTimePoint_n$ iff, for all $0\leq i\leq m$, we have that $\happeningActs_i\subseteq \bigcup \{\{\langle\durActionInv,\emptyset,\emptyset\rangle,\durActionStart,\durActionEnd\}\mid\langle\durAction,t,d\rangle\in\plan\}$ and, for all $\langle\durAction,t,d\rangle\in\plan$, 
  \begin{enumerate*}
    \item there is a happening $\langle\happeningActs_i,\happeningTime_i\rangle$ with $\happeningTime_i=t$ and $\durActionStart\in\happeningActs_i$,
    \item there is a happening $\langle\happeningActs_j,\happeningTime_j\rangle$ with $\happeningTime_j=t+d$ and $\durActionEnd\in\happeningActs_j$,
    \item\label{item:ind_happ_seq_iii} for each $0\leq l<n$ with $t\leq\happTimePoint_l<t+d$ there is  $\langle\happeningActs_k,\happeningTime_k\rangle$ with $\happTimePoint_l<\happeningTime_k<\happTimePoint_{l+1}$ and $\langle\durActionInv,\emptyset,\emptyset\rangle\in\happeningActs_k$, and
    \item the starting time points $\happeningTime_0,\dots,\happeningTime_m$ are strictly sorted in an ascending order.
  \end{enumerate*}
\end{mydef}

\lukas{Do we need the extra conjunct in this definition? Every happening only contains correct snap actions for the plan. YES!}

\begin{figure}
  \centering
  \begin{tikzpicture}[label distance=0mm,scale=1.15]  
  \draw (0,1.4) -- (0,1.6);
  \draw (0,1.5) node[label={[yshift=-1mm]above:{\tiny $\mathstrut 0$}}] {} -- (3,1.5);
  \draw (3,1.4) -- (3,1.6);
  \draw (3,1.5) node[label={[yshift=-1mm]above:{\tiny $\mathstrut 0.75$}}] {} -- (4,1.5);
  \draw (4,1.4) -- (4,1.6);
  \draw (4,1.5) node[label={[yshift=-1mm]above:{\tiny $\mathstrut 1$}}] {} -- (5,1.5);
  \draw (5,1.4) -- (5,1.6);
  \draw (5,1.5) node[label={[yshift=-1mm]above:{\tiny $\mathstrut 1.25$}}] {} -- (6,1.5);
  \draw (6,1.4) -- (6,1.6);
  \draw (6,1.5) node[label={[yshift=-1mm]above:{\tiny $\mathstrut 1.5$}}] {} -- (6.5,1.5);
  \draw[dotted] (6.5,1.5) -- (7.0,1.5);

  \draw (5.5,1.45) -- (5.5,1.55);
  \draw (4.5,1.45) -- (4.5,1.55);
  \draw (3.5,1.45) -- (3.5,1.55);
  \draw (1.5,1.45) -- (1.5,1.55);

  \filldraw[fill=green!80!black!15!white,draw=green!80!black!75!white,line width=0.2mm] (0,0.85) rectangle (4,1.35) {};
  \fill[fill=yellow!60!red!15!white] (5,0.85) rectangle (6.5,1.35);
  \draw[draw=yellow!60!red!75!white,line width=0.2mm] (5,0.85) -- (6.5,0.85);
  \draw[draw=yellow!60!red!75!white,line width=0.2mm] (5,0.85) -- (5,1.35);
  \draw[draw=yellow!60!red!75!white,line width=0.2mm] (5,1.35) -- (6.5,1.35);
  \draw[dotted] (6.4,1.1) -- (6.75,1.1);
  \filldraw[fill=yellow!60!red!15!white,draw=yellow!60!red!75!white,line width=0.2mm] (3,0.2) rectangle (5,0.7);
  \fill[fill=green!80!black!15!white] (6,0.2) rectangle (6.5,0.7);
  \draw[draw=green!80!black!75!white,line width=0.2mm] (6,0.2) -- (6.5,0.2);
  \draw[draw=green!80!black!75!white,line width=0.2mm] (6,0.2) -- (6,0.7);
  \draw[draw=green!80!black!75!white,line width=0.2mm] (6,0.7) -- (6.5,0.7);
  \draw[dotted] (6.4,0.45) -- (6.75,0.45);

  \node at (2,1.1) {\tiny $\mathstrut \egactionaTiny$};
  \node at (5.75,1.1) {\tiny $\mathstrut \egactionbTiny$};
  \node at (4,0.45) {\tiny $\mathstrut \egactioncTiny$};
  \node at (6.25,0.45) {\tiny $\mathstrut \egactiondTiny$};

  \filldraw[fill=red!30!white,draw=red!90!white,line width=0.2mm] (5.875,-0.125) rectangle (6.125,-0.375);
  \node[label={below:{\tiny $\mathstrut \happening_9$}}] at (6,-0.25) {};
  \filldraw[fill=red!15!white,draw=red!75!white,line width=0.2mm] (5.375,-0.125) rectangle (5.625,-0.375);
  \node[label={below:{\tiny $\mathstrut \happening_8$}}] at (5.5,-0.25) {};
  \filldraw[fill=red!30!white,draw=red!90!white,line width=0.2mm] (4.875,-0.125) rectangle (5.125,-0.375);
  \node[label={below:{\tiny $\mathstrut \happening_7$}}] at (5,-0.25) {};
  \filldraw[fill=red!15!white,draw=red!75!white,line width=0.2mm] (4.375,-0.125) rectangle (4.625,-0.375);
  \node[label={below:{\tiny $\mathstrut \happening_6$}}] at (4.5,-0.25) {};
  \filldraw[fill=red!30!white,draw=red!90!white,line width=0.2mm] (3.875,-0.125) rectangle (4.125,-0.375);
  \node[label={below:{\tiny $\mathstrut \happening_5$}}] at (4,-0.25) {};
  \filldraw[fill=red!15!white,draw=red!75!white,line width=0.2mm] (3.375,-0.125) rectangle (3.625,-0.375);
  \node[label={below:{\tiny $\mathstrut \happening_4$}}] at (3.5,-0.25) {};
  \filldraw[fill=red!30!white,draw=red!90!white,line width=0.2mm] (2.875,-0.125) rectangle (3.125,-0.375);
  \node[label={below:{\tiny $\mathstrut \happening_3$}}] at (3,-0.25) {};
  \filldraw[fill=red!15!white,draw=red!75!white,line width=0.2mm] (1.375,-0.125) rectangle (1.625,-0.375);
  \node[label={below:{\tiny $\mathstrut \happening_2$}}] at (1.5,-0.25) {};
  \filldraw[fill=red!30!white,draw=red!90!white,line width=0.2mm] (-0.125,-0.125) rectangle (0.125,-0.375);
  \node[label={below:{\tiny $\mathstrut \happening_1$}}] at (0,-0.25) {};
  \draw[dotted] (6.4,-0.25) -- (6.75,-0.25);

  \draw[dotted] (6,1.4) -- (6,-0.25);
  \draw[dotted] (5.5,1.4) -- (5.5,-0.25);
  \draw[dotted] (5,1.4) -- (5,-0.25);
  \draw[dotted] (4.5,1.4) -- (4.5,-0.25);
  \draw[dotted] (4,1.4) -- (4,-0.25);
  \draw[dotted] (3.5,1.4) -- (3.5,-0.25);
  \draw[dotted] (3,1.4) -- (3,-0.25);
  \draw[dotted] (1.5,1.4) -- (1.5,-0.25);
  \draw[dotted] (0,1.4) -- (0,-0.25);

  \node[label={[xshift=-3mm]right:{\tiny $\mathstrut \happening_1\equiv\langle\{\egactionaTiny_\textit{start}\},0.0\rangle$}}] at (0.0,-1) {};
  \node[label={[xshift=-3mm]right:{\tiny $\mathstrut \happening_2\equiv\langle\{\egactionaTiny_\textit{inv}\},0.375\rangle$}}] at (0.0,-1.3) {};
  \node[label={[xshift=-3mm]right:{\tiny $\mathstrut \happening_3\equiv\langle\{\egactioncTiny_\textit{start}\},0.75\rangle$}}] at (0.0,-1.6) {};
  \node[label={[xshift=-3mm]right:{\tiny $\mathstrut \happening_4\equiv\langle\{\egactionaTiny_\textit{inv},\egactioncTiny_\textit{inv}\},0.875\rangle$}}] at (0.0,-1.9) {};
  \node[label={[xshift=-3mm]right:{\tiny $\mathstrut \happening_5\equiv\langle\{\egactionaTiny_\textit{end}\},1.0\rangle$}}] at (0.0,-2.2) {};

  \node[label={[xshift=-3mm]right:{\tiny $\mathstrut \happening_6\equiv\langle\{\egactioncTiny_\textit{inv}\},1.125\rangle$}}] at (4,-1) {};
  \node[label={[xshift=-3mm]right:{\tiny $\mathstrut \happening_7\equiv\langle\{\egactioncTiny_\textit{end},$}}] at (4,-1.3) {};
  \node[label={[xshift=-3mm]right:{\tiny $\mathstrut \egactioncTiny_\textit{start}\},1.25\rangle$}}] at (4.85,-1.6) {};
  \node[label={[xshift=-3mm]right:{\tiny $\mathstrut \happening_8\equiv\langle\{\egactionbTiny_\textit{inv}\},1.375\rangle$}}] at (4,-1.9) {};
  \node[label={[xshift=-3mm]right:{\tiny $\mathstrut \happening_9\equiv\langle\{\egactiondTiny_\textit{start}\},1.5\rangle$}}] at (4,-2.2) {};
\end{tikzpicture}
   \caption{Illustration for the beginning of an induced happening sequence (Def.~\ref{def:ind_happ_seq}) for the elevator-running example. 
  }\label{fig:ind_happ_seq_illustration}
\end{figure}
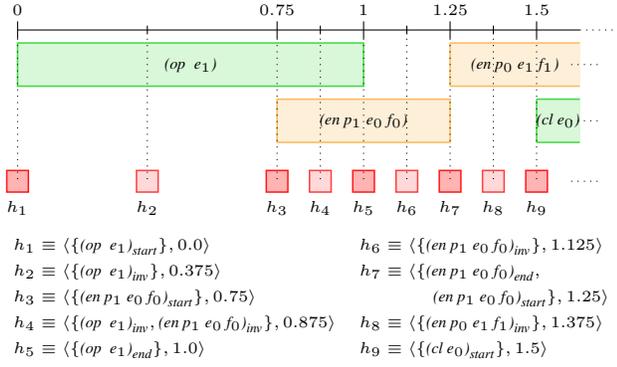

Figure \ref{fig:ind_happ_seq_illustration} illustrates the beginning of an induced happening sequence for the elevator running example.
At the top of the illustration, a timeline with the happening time points is shown. Every start- or end-point of a plan action is a happening time point.
In this example, the first five happening time points are: $0, 0.75, 1, 1.25,$ and $1.5$. Below the timeline the first four actions from the valid plan are shown. For each plan action, the snap actions are placed along the timeline and collected in the happenings, which are symbolized as red squares in the illustration.
E.g.\ for the first plan action $\langle\egactiona,0,1\rangle$, the start snap action $\egactiona_\textit{start}$ is placed at the start of the action, at time point $0$ and collected in happening $h_1$, whereas the end snap action $\egactiona_\textit{end}$ is placed at the end of the action, at time point $1$ and collected in happening $h_5$.
For every two consecutive happening time points the invariants of all currently running actions need to be checked.
Therefore, happening $h_2$ contains the invariant snap action for the first plan action $\egactiona$.
In between the consecutive happening time points $0.75$ and $1$ the action $\egactiona$ is running as well as the action $\egactionc$, hence the happening $h_4$ contains the invariant snap actions for both $\egactiona$ and $\egactionc$.

The illustration in Figure~\ref{fig:ind_happ_seq_illustration} only shows one possible induced happening sequence for the valid plan.
Def.~\ref{def:ind_happ_seq} allows invariant snap actions to be placed arbitrarily in between consecutive happening time points.
This is more general than the definition of \citeauthor{fox2003pddl2}, which arbitrarily restricts the placement of invariant snap actions to be exactly in the middle of happening time points. 
We use this placement of invariant actions in the next section, where we give an executable definition of plan validity.
Based on the notion of valid happening we define the following notion of plan validity, which is closer to the definition of \citeauthor{fox2003pddl2} and to executability.

\begin{mydef}[Valid Plan II]\label{def:val_plan}
  \label{def:plan_validity_original}
  Plan $\plan$ is valid for a planning problem $\planningproblem$ iff $\plan$ has an induced happening sequence $\happening_0\dots,\happening_n$ s.t.\ the happening sequence is valid wrt $\initState$ and $\planningState_{n+1}\models\goalState$, where $\planningState_{n+1}$ is the last state in the induced state sequence.
\end{mydef}

At a higher-level, the contrast between Def.~\ref{def:plan_validity_original}~and~\ref{def:plan_validity_rec} boils down to that the former specifies plan validity in terms of a happening sequence that should be computed, while the latter specifies validity more abstractly.
More specifically, instead of referring to happening sequences, Def.~\ref{def:plan_validity_rec} uses $\recSetActs_\happTimePoint$ and $\recSetInvs_\happTimePoint$, which denote the snap actions executing at time $t$ and the set of invariants which should hold at time $t$, respectively.
Accordingly, for Def.~\ref{def:plan_validity_original} we only assert the existence of a sequence of valid states, which can be formalised, in Isabelle/HOL, as a simple recursion on the happening time points of a plan, instead of asserting the existence of an induced happening sequence as in the case of Def.~\ref{def:plan_validity_original}.
The two definitions are equivalent as shown below.

\begin{mythm}\label{thm:sema_equiv}
  For a planning problem $\planningproblem$, a plan $\as$ is valid according to Def.~\ref{def:plan_validity_original} iff it is valid according to Def.~\ref{def:plan_validity_rec}.
\end{mythm}
\begin{proof}[Proof sketch]
  Let $\happTimePoint_0,\dots,\happTimePoint_n$ be the happening time points of $\as$ after being sorted in ascending order.

  \noindent($\Rightarrow$) From Def.~\ref{def:plan_validity_original}, $\as$ has an induced happening sequence $\langle\happeningActs_0,\happeningTime_0\rangle,\dots,\langle\happeningActs_m,\happeningTime_m\rangle$, and that happening sequence is valid wrt $\initState$. Note that $m\geq n$.
  Our goal here is to show that the induced state sequence of this happening sequence is a valid state sequence, according to Def.~\ref{def:rec_sets}.
  Since the induced happening sequence is strictly sorted according to the starting time of the happenings, we know that the different happenings have different starting points.
  Accordingly, we have, for each $\happTimePoint_i$, where $0\leq i \leq n$, there is a happening $\langle\happeningActs_j,\happeningTime_j\rangle$, s.t.\ $\recSetActs_{\happTimePoint_i}=\happeningActs_j$ and $\happTimePoint_i=\happeningTime_j$.
  Since this induced happening sequence is also a valid happening sequence, the conjuncts (ii), (iii), and (iv) of Def.~\ref{def:rec_sets} hold for the induced state sequence.
  What remains is to show that conjunct (i) holds for the induced state sequence, which states that all action invariants hold during action execution.
  Observe that conjunct (iii) of Def.~\ref{def:ind_happ_seq} asserts that, for each $\langle\act,t,d\rangle\in\as$, there is an action $\langle\durActionInv,\emptyset,\emptyset\rangle$ between each two happenings that happen during the execution of an action $\act$.
  The preconditions of this action ensure that the invariants of the action $\act$ are not violated during its execution.
  Accordingly, conjunct (i) holds for the induced state sequence.

  \noindent($\Leftarrow$) To prove this direction, we need to show that $\as$ has an induced happening sequence, which is valid wrt $\initState$, from a given valid state sequence $\initState,\planningState_1,\dots,\planningState_{n+1}$.
  Consider the happening sequence $\langle\recSetActs_{\happTimePoint_0},\happTimePoint_0\rangle,$ $\langle\recSetInvs_{\happTimePoint_1},\frac{\happTimePoint_0+\happTimePoint_1}{2}\rangle,$ $\langle\recSetActs_{\happTimePoint_1},\happTimePoint_1\rangle,$ $\langle\recSetInvs_{\happTimePoint_2},\frac{\happTimePoint_1+\happTimePoint_2}{2}\rangle,\dots,\langle\recSetActs_{\happTimePoint_{n-1}},\happTimePoint_{n-1}\rangle,\langle\recSetInvs_{\happTimePoint_n},\frac{\happTimePoint_{n-1}+\happTimePoint_n}{2}\rangle, \langle\recSetActs_{\happTimePoint_{n}},\happTimePoint_{n}\rangle$.
  We now need to show that this happening sequence is a valid one, according to Def.~\ref{def:val_happ_seq}.
  It is easy to see that conjunct (ii) of Def.~\ref{def:val_happ_seq} holds for this happening sequence.
  To show that the other two conjuncts of Def.~\ref{def:val_happ_seq} hold, we first need to provide a witness state sequence to which those conjuncts apply.
  The state sequence $\initState,\planningState_1,\planningState_1,\dots,\planningState_{n+1},\planningState_{n+1}$\footnote{This repetition of states is intended: each state $\planningState_i$ occurs first as a result of executing the happening $\langle\recSetActs_{\happTimePoint_{i}},\happTimePoint_{i}\rangle$ at state $\planningState_{i-1}$ and then second as a result of executing the happening $\langle\recSetInvs_{\happTimePoint_i},\frac{\happTimePoint_{i-1}+\happTimePoint_{i}}{2}\rangle$, which has no effects, at state $\planningState_i$.} is the witness:
  \begin{itemize*}
    \item Conjunct (i) of Def.~\ref{def:val_happ_seq} holds for $\initState,\planningState_1,\planningState_1,\dots,\planningState_{n+1},\planningState_{n+1}$ because conjunct (i) of Def.~\ref{def:rec_sets} holds for $\initState,\planningState_1,\dots,\planningState_{n+1}$, which implies that the preconditions in each action in a happening $\langle\recSetActs_{\happTimePoint_{i}},\happTimePoint_{i}\rangle$ are entailed by the state $\planningState_i$, and conjunct (ii) of Def.~\ref{def:rec_sets} also holds for $\initState,\planningState_1,\dots,\planningState_{n+1}$, which implies that the preconditions of each happening $\langle\recSetInvs_{\happTimePoint_i},\frac{\happTimePoint_{i-1}+\happTimePoint_{i}}{2}\rangle$ are entailed by the state $\planningState_{i-1}$.
    \item Conjunct (iii) of Def.~\ref{def:val_happ_seq} holds for $\initState,\planningState_1,\planningState_1,\dots,\planningState_{n+1},\planningState_{n+1}$ because conjunct (iii) of Def.~\ref{def:val_happ_seq} holds for $\initState,\planningState_1,\dots,\planningState_{n+1}$.
  \end{itemize*}
  The last remaining thing is to show that the happening sequence we constructed is an induced happening sequence for $\as$, according to Def.~\ref{def:ind_happ_seq}:
  \begin{itemize*}\item The first two conjuncts of Def.~\ref{def:ind_happ_seq} hold for this happening sequence because from the definition of $\recSetActs$ and $\recSetInvs$.
    \item The third conjunct holds due to the way we construct the happening sequence.
    \item The fourth conjunct holds because we have the happening time points already sorted and the way we construct our happening sequence.
  \end{itemize*}
  This finishes our proof.
\end{proof}

\mohammad{Abstract notes about the Isabelle/HOL proof.}
\mohammad{Any philosophical implications/conclusions (cite ASU thesis).}

\mohammad{Before we close this section, we would like to make a few notes comparing our definition of an induced happening sequence and that of \citeauthor{fox2003pddl2}.
\citeauthor{fox2003pddl2} describe the induced happening sequence by a construction algorithm.
We formalise the induced happening sequence with a predicate which characterises the properties of an induced happening sequence.
Therefore, we need to add an additional constraint to the formalisation of an induced happening sequence: we need to make sure that all ground actions which appear in an induced happening sequence actually are a grounded and instantiated snap action for a plan action appearing the given plan.
In the semantics by \citeauthor{fox2003pddl2} this constraint is implicitly asserted since they use a construction algorithm.}

 \section{An Executable Verified Validator}
\label{sec:}

\SetKwIF{If}{ElseIf}{Else}{if}{}{else if}{else}{endif}
\begin{algorithm}[t!]\DontPrintSemicolon
  \KwSty{function} $\insertActionAlgo(\langle\happeningActs_0,\happeningTime_0\rangle,\dots,\langle\happeningActs_m,\happeningTime_m\rangle,t,\act)$\;
  \ \ \KwSty{for each} $0\leq i < m$\;
  \ \ \ \ \KwSty{if} $\happeningTime_i = t$\;
  \ \ \ \ \ \ \ \ \KwSty{ret} $\langle\happeningActs_0,\happeningTime_0\rangle,\dots,\langle\happeningActs_i\cup\{\act\},\happeningTime_i\rangle,\dots,\langle\happeningActs_m,\happeningTime_m\rangle$\;
  \ \ \ \ \KwSty{if} $\happeningTime_{i+1} = t$\;
  \ \ \ \ \ \ \ \ \KwSty{ret} $\langle\happeningActs_0,\happeningTime_0\rangle,\dots,\langle\happeningActs_{i+1}\cup\{\act\},\happeningTime_{i+1}\rangle,\dots,$\;\ \ \ \ \ \ \ \ \ \ \ \ \ \ \ \ \ $\langle\happeningActs_m,\happeningTime_m\rangle$\;
  \ \ \ \ \KwSty{if} $\happeningTime_i < t < \happeningTime_{i+1}$\;
  \ \ \ \ \ \ \ \ \KwSty{ret} $\langle\happeningActs_0,\happeningTime_0\rangle,\dots,\langle\happeningActs_i,\happeningTime_i\rangle,\langle\{\act\},\happeningTime\rangle,$\; \ \ \ \ \ \ \ \ \ \ \ \ \ \ \ \ \ $\langle\happeningActs_{i+1},\happeningTime_{i+1}\rangle,\dots,\langle\happeningActs_m,\happeningTime_m\rangle$\;
  \vspace{0.5ex}
  \KwSty{function} $\simplifyActionAlgo(\happTimePoint_0,\dots\happTimePoint_n, \langle\act,t,d\rangle,\happSeq)$\;
  \ \ $\happSeq:=\insertActionAlgo(H,t,\durActionStart)$\;
  \ \ $\happSeq:=\insertActionAlgo(H,t+d,\durActionEnd)$\;
  \ \ \KwSty{for each} $0\leq i<n$\;
  \ \ \ \ \KwSty{if} $t\leq\happTimePoint_i < \happTimePoint_{i+1}\leq t+d$\;
  \ \ \ \ \ \ \ \ $\happSeq:=\insertActionAlgo(\happSeq,\frac{\happeningTime_i+\happeningTime_j}{2},\langle\durActionInv,\emptyset,\emptyset\rangle)$\;
  \ \ \KwSty{ret} $\happSeq$\;
  \vspace{0.5ex}
  \KwSty{function} $\simplifyPlanAlgo(\plan)$\;
  \ \ $\happSeq:=\emptyset$\;
  \ \ \KwSty{for each} $\langle\act,t,d\rangle\in\plan$\;
  \ \ \ \ $\simplifyActionAlgo(\happTimePoints(\plan),\langle\act,t,d\rangle,\happSeq)$\;
  \ \ \KwSty{ret} $\happSeq$\;
  \vspace{0.5ex}
  \KwSty{function} $\validHapSeqAlgo(\langle\happeningActs_0,\happeningTime_0\rangle,\dots,\langle\happeningActs_m,\happeningTime_m\rangle, \planningproblem)$\;
  \ \ $\planningState:=\init$\;
  \ \ \KwSty{for each} $0\leq i\leq m$\;
  \ \ \ \ \KwSty{if} $\exists\acta,\actb\in\happeningActs_i$ \KwSty{and} they are interfering\;
  \ \ \ \ \ \ \ \ \KwSty{ret} \false\;
  \ \ \ \ \KwSty{if} $\exists\act\in\happeningActs_i. \planningState\not\models\snapActionPre$\;
  \ \ \ \ \ \ \ \ \KwSty{ret} \false\;
  \ \ \ \ $\planningState:= \left(\planningState-\bigcup_{\snapAction\in\happeningActs_i}\snapActionDel\right)\cup\bigcup_{\snapAction\in\happeningActs_i}\snapActionAdd$\;
  \ \ \KwSty{if} $\planningState\models\goal$\;
  \ \ \ \ \KwSty{ret} $\true$ \;
  \ \ \KwSty{ret} $\false$ \;
  \vspace{0.5ex}
  \KwSty{function} $\checkPlanAlgo(\planningproblem,\plan)$\;
  \ \ $\happSeq:=\simplifyPlanAlgo(\plan)$\;
  \ \ \KwSty{if} $\validHapSeqAlgo(\happSeq,\planningproblem)$\;
  \ \ \ \ \KwSty{ret} ``valid Plan'' \;
  \ \ \KwSty{ret} ``error''\;
  \caption{The executable specification of plan validity, $\checkPlanAlgo$, as pseudo-code.
    In this pseudo-code, $\planningproblem$ denotes a planning problem, $\protect\plan$ a plan to be checked, $\happSeq$ a sequence of happenings, $\happeningActs_i$ a set of snap actions, $\happeningTime_i$ a happening starting time point, and $\happTimePoint$ a happening time point.
  }\label{alg:validplan}
\end{algorithm}

The last part of our work is regarding implementing an executable specification of the semantics, i.e.\ a plan validation algorithm, and formally proving that it is equivalent to the unexecutable specification of the semantics in Def.~\ref{def:plan_validity_rec}.
The formalized semantics are defined with unexecutable abstract mathematical types and depend on several mathematical concepts, e.g. sets and quantifiers.
To obtain an executable validator these mathematical types and concepts need to be replaced with efficient algorithms.
We use step-wise refinement to replace the abstract specifications in the semantics with algorithms.
With step-wise refinement efficient implementations of algorithms can be proven correct by using multiple correctness preserving steps to refine an abstract version of the algorithm towards the efficient implementation.
This allows us to formalize concise semantics and implement an efficient validator wrt.\ those semantics.

We do two main refinement steps: first, we replace the abstract specifications of the semantics with algorithms defined on abstract mathematical types like sets.
This is shown in the pseudo-code of our validation algorithm in Algorithm~\ref{alg:validplan}, where $\checkPlanAlgo$ is the top-level routine.
We then prove the following theorem about it.

\begin{mythm}\label{thm:validator}
  $\checkPlanAlgo(\planningproblem,\plan)=\text{"valid Plan"}$ iff $\plan$ is valid for the planning problem $\planningproblem$ according to Def.~\ref{def:plan_validity_rec}.
\end{mythm}

\begin{mylem}\label{lem:happseqeq}
  Let $\plan$ be a plan and $\happSeq$ and $\happSeq'$ be induced happening sequences for $\plan$.
  If a state sequence is an induced state sequence by a state $\planningState_0$ and $\happSeq$, then there is a state sequence induced by $\planningState_0$ and $\happSeq'$, where the last state of the two sequences is the same.
\end{mylem}
\begin{proof}[Proof sketch]
Firstly, let $\happSeq$ ($\happSeq'$) be $(\happeningActs_0,\happeningTime_0),(\happeningActs_1,\happeningTime_1),\dots,$ $(\happeningActs_m,\happeningTime_m)$ ($(\happeningActs_0',\happeningTime_0),(\happeningActs_1',\happeningTime_1'),\dots, (\happeningActs_{m'},\happeningTime_{m'})$), let $\happTimePoint_0,\happTimePoint_1,\dots,\happTimePoint_n$ be the happening time points of $\plan$, and let $\planningState_1,\planningState_2,\dots,\planningState_{m+1}$ ($\planningState_1',\planningState_2',\dots,\planningState_{m'+1}$) be the induced state sequences of $\init$ and $\happSeq$ ($\happSeq'$).
Because of the fourth conjunct of Def.~\ref{def:ind_happ_seq}, we have a monotonically increasing mapping $f$ ($f'$) from $\{0,1,\dots,n\}$ to $\{0,1,\dots,m\}$ ($\{0,1,\dots,m'\}$), such that, for $0\leq i \leq n$, $t_i = r_{f(i)}$ ($t_i = r_{f'(i)}$) and $f(n) = m$ ($f'(n) = m'$).
Also, from the third conjunct of Def.~\ref{def:ind_happ_seq} we have that, for $0\leq i \leq n$, $\happeningActs_{f(i)}$ ($\happeningActs_{f'(i)}$) has no invariant snap actions and, accordingly, $\happeningActs_{f(i)} = \happeningActs_{f'(i)}$, and for $j\in\{0,1,\dots,m\}\setminus\{f(0),f(1),\dots,f(n)\}$ ($j\in\{0,1,\dots,m\}\setminus\{f(0),f(1),\dots,f(n)\}$), $\happeningActs_j$ has only invariant snap actions, i.e.\ $\happeningActs_j \subseteq\{\langle\phi,\emptyset,\emptyset\rangle\mid\phi$ is propositional formula$\}$.
From the two previous statements, we conclude that $\planningState_{f(i)} = \planningState_{f'(i)}$, for $0 < i \leq n$, which finishes our proof.
\end{proof}

\begin{mylem}\label{lem:simplify}
  For any plan $\plan$, $\simplifyPlanAlgo(\plan)$ is an induced happening sequence for the plan $\plan$.
\end{mylem}
\begin{proof}[Proof sketch]
This follows from Def.~\ref{def:ind_happ_seq}.
\end{proof}

\begin{mylem}\label{lem:validhappseq}
  For any happening sequence $\happSeq$ and planning problem $\planningproblem$, $\validHapSeqAlgo(\happSeq,\planningproblem)$ is true iff $\happSeq$ is a valid happening sequence wrt.\ $\init$.
\end{mylem}
\begin{proof}[Proof sketch]
This follows from Def.~\ref{def:val_happ_seq}.
\end{proof}

\begin{proof}[Proof of Theorem~\ref{thm:validator}]
  The theorem follows from Lemmas~\ref{lem:simplify},~\ref{lem:validhappseq}~and~\ref{lem:happseqeq}, and Theorem~\ref{thm:sema_equiv}.
\end{proof}

A validator has to be executable and efficient and thus the implementation of a validator is more complicated than the formalisation of the semantics.

In the next step-wise refinement step, the abstract mathematical types, like the set operations in $\validHapSeqAlgo$, are replaced with efficient implementation using balanced trees.
Since this step is completely automated with the Containers Framework in Isabelle/HOL~\cite{lochbihler2013light}, we do not describe the resulting pseudo-code or the proofs of its equivalence to the pseudo-code from Algorithm~\ref{alg:validplan}.

Before we close this section we would like to note two points.
First, the formal version of Algorithm~\ref{alg:validplan} includes checks related to PDDL-level well-formedness, like the correctness of typing of action arguments, etc.
These details are similar to what was done by Abdulaziz and Lammich and we ignore them here as we only focus on grounded problems.
Readers interested in the PDDL-level reasoning can consult the associated formalisation.
Second, as one of our goals was to simplify the semantics, we do not assert the presence of a concrete minimum separation, $\epsilon$, between plan actions.
In our refinement steps, we are able to derive a validation algorithm which uses arbitrary arithmetic on rational numbers and it is formally proved to implement Def.~\ref{def:plan_validity_rec}.
This is an improvement over the approach of~\citeauthor{fox2003pddl2}, who claimed in their paper that it is necessary to accept that numeric conditions, including time, will have to be evaluated to a certain tolerance.
Indeed, VAL~\cite{howey2004val} implements this $\epsilon$ and thus requires the $\epsilon$ as an extra parameter.
This leads to rejecting, otherwise valid, plans if a too large $\epsilon$ is given to VAL.

\lukas{Here describe the validator pen-and-paper.}
\lukas{Describe the proof that the validator implements the semantics.}
\lukas{Describe the fact that arbitrary precision needs no $\epsilon$}

\lukas{Copy\&Paste from thesis: Might reuse?}

\lukas{The semantics require a minimum time separating interfering ground actions. The formalized semantics allow for an arbitrarily small minimum time step.
  Arbitrary-precision numbers usually pose a challenge for computers and programs.
  Therefore, \citeauthor{fox2003pddl2} mention in \cite{fox2003pddl2} the use of a fixed minimum time step, called $\epsilon$, for the semantics to make automated planning and plan validation in temporal PDDL easier.
  \citeauthor{fox2003pddl2} do not specify a concrete value for $\epsilon$.
  The well-known validator VAL \cite{howey2004val} implements this $\epsilon$ and thus requires an extra parameter specifying the value of the $\epsilon$.
  When using the $\epsilon$ time separation seemingly valid plans might be rejected solely due to a too large $\epsilon$ value.
  The implementation of our validator uses arbitrary-precision numbers.
  Any specified number in a PDDL-domain, -problem, or -plan is a finite decimal number.
  All numbers are parsed as rational numbers without loss of any precision.
  Therefore, the plan validity in our validator does not need to rely on a fixed $\epsilon$ value.}

\lukas{We show $\epsilon$ is not a pragmatic choice (see Fox \& Long).}

\paragraph{Parsing Problems and Code Generation}
\label{sec:parsing}

For parsing, we use an open source parser combinator library written in Standard ML.
We note that parsing is a trusted part of our validator, i.e.\ we have no formal proof that the parser actually recognises the desired grammar and produces the correct abstract syntax tree.
However, the parsing combinator approach allows to write concise, clean, and legible parsers, which can be relatively easily checked.

\paragraph{Experimental Evaluation}
\label{sec:experiments}

\mohammad{Search for bugs in VAL: fuzzing?}
\lukas{Describe the setup, performance and any bugs encountered in VAL.}

\lukas{CPU and RAM?? Done.}

Our validator supports the following PDDL requirements: \mbox{\lstinline[style=isainline]{:strips}}, \mbox{\lstinline[style=isainline]{:equality}}, \mbox{\lstinline[style=isainline]{:typing}}, \mbox{\lstinline[style=isainline]{:negative-preconditions}}, \mbox{\lstinline[style=isainline]{:disjunctive-preconditions}}, \mbox{\lstinline[style=isainline]{:durative-actions}}, and \mbox{\lstinline[style=isainline]{:duration-inequalities}}.
For the evaluation of our validator, we compare the validation results and running time of our validator to those of VAL~\cite{howey2004val}.
We use IPC 2014 domains.
We used the temporal planners ITSAT~\cite{rankooh_ghassemsani_2015} and Temporal Fast Downward (TFD)~\cite{eyerich_mattmueller_roeger_2009} to generate plans for the domains and problems. 
In all test cases, the validation outcome between our validator and VAL is the same.
Our validator is consistently slower than VAL, as can be seen in Figure~\ref{fig:val-perf-comp}.
However, it never needs more than one second to validate any plan.
This is a practically acceptable performance, escpecially since our validator uses arbitrary precision arithmetic.
We also note that formally verified code is usually orders of magnitude slower than unverified code due to the difficulty of verifying all code optimisations which are liberally used in unverified code.
\lukas{profiling code. }

\begin{figure}
  \centering
  \begin{tikzpicture}
    \begin{axis}[
      scatter/classes={
          val={mark=o,draw=blue},
          our={mark=x,draw=red}},
      legend pos=north west,
      ylabel=runtime {[s]},
      axis x line*=bottom,
      xticklabels={crewplanning, elevators, openstacks, pegsol, sokoban, Driverlog, MatchCellar, Parking, Satellite, Storage, TurnandOpen},
      xtick={1,...,11},
      x tick label style={rotate=90,anchor=east,font=\tiny},
      height=5cm, width=\linewidth]
      \addplot[scatter,only marks,
        scatter src=explicit symbolic]
      table[meta=label] {test_out_scatter_lrz.txt};
      \legend{VAL, our validator}
    \end{axis}
  \end{tikzpicture}
  \caption{Validation running times for IPC 2014 domains.}\label{fig:val-perf-comp}
\end{figure}
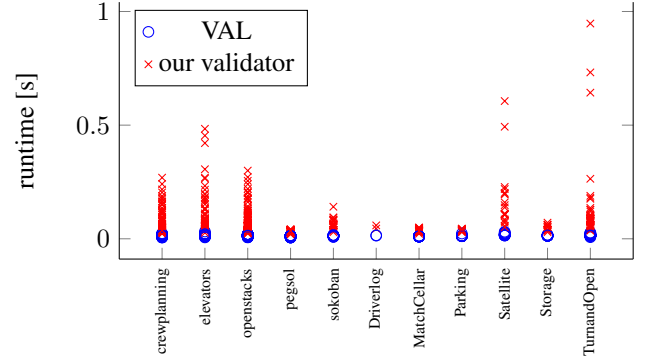

 \section{Discussion}
\label{sec:conclusion}

In this work we presented the first specification of the semantics of the temporal part of PDDL2.1 in a formal mathematical system, namely, Isabelle/HOL.
Specifying language semantics in formal mathematical systems has the advantages of removing any ambiguities and providing the basis to build formally verified tool chains to reason about these languages.
These advantages of formalising language semantics have been reported by researchers who use ITPs to formalise programming language semantics, e.g.\ C~\cite{norrish1998c}, SML~\cite{DBLP:conf/popl/KumarMNO14}, and Rust~\cite{DBLP:journals/pacmpl/0002JKD18}.
One main purpose of our work was to showcase the merits of this methodology to the planning community.

The semantics and validation of the temporal fragment of PDDL have been studied by multiple authors.
We believe our work improves over all the previous approaches in two aspects: the \emph{succinctness} of our semantics specificaiton and the \emph{trusworthiness} of our executable validator.

PDDL2.1 was first introduced during the second international planning competition and its semantics were most comprehensively defined by~\citeauthor{fox2003pddl2}~\citeyear{fox2003pddl2}.
We base our work on the semantics of~\citeauthor{fox2003pddl2}.
One issue with their semantics noted by earlier authors~\citeauthor{DBLP:conf/aaai/ClassenHL07} is that it defines plan validity using an executable plan validation algorithm, which is more complicated than what a specification of semantics ought to be.
We address that by providing simpler semantics and showing it is equivalent to an executable validator.
Our semantics are simpler because they \begin{enumerate*}\item remove the need for a fixed ``$\epsilon$'' separation between interfering actions, requiring only an arbitrary non-zero separation, \item bypass the concept of induced happening sequences, and \item do not require that snap actions representing invariants occur exactly between each two happenings which occur while the invariant has to hold.\end{enumerate*}
Another difference between our work and that of~\citeauthor{fox2003pddl2} is that we specify our semantics in Isabelle/HOL wrt abstract syntax which is very close to PDDL syntax.\footnote{Interested readers should consult the formalisation.}
This gives rise to a more detailed specification of the semantics and leaves less room for ambiguities.

Another tangentially related work is that of \citeauthor{DBLP:conf/aaai/GiganteMMS20}~\citeyear{DBLP:conf/aaai/GiganteMMS20}.
In their work, they studied the complexity of computing plans for different restrictions of the temporal planning as described by~\citeauthor{fox2003pddl2}.

Another notable planning language which includes temporal elements is ANML~\cite{smith2008anml}.
The semantics of a language ``inspired'' by ANML were defined by \citeauthor{DBLP:conf/aaai/CimattiMR17}~\citeyear{DBLP:conf/aaai/CimattiMR17}.
Although \citeauthor{DBLP:conf/aaai/CimattiMR17} use pen-and-paper definitions, the level of detail of their presentation is closer to ours as they specified an abstract syntax for their language, based on which they defined their semantics.
However, our semantics are much more succinct than theirs since we use HOL to specify our semantics, while they specify their semantics in terms of linear temporal logic modulo real arithmetic, which is significantly less expressive than HOL.

Another well-established formalism for studying the semantics of planning and action languages in general is situation calculus~\cite{mccarthy1981some,reiterSitCalc}.
In that line of work, the work by \citeauthor{DBLP:conf/aaai/ClassenHL07}~\citeyear{DBLP:conf/aaai/ClassenHL07} is the most related to this paper.
They showed how to encode a PDDL 2.1 problem as a formula in $\mathcal{ES}$, which is a dialect of first-order logic with interesting computational and meta-theoretic properties introduced by~\citeauthor{DBLP:conf/kr/LakemeyerL04}~\citeyear{DBLP:conf/kr/LakemeyerL04}.
The main merit of that approach, as stated by~\citeauthor{DBLP:conf/aaai/ClassenHL07}, is that their semantics are a declarative specification of the semantics of PDDL 2.1 as opposed to the state transition-based semantics of~\citeauthor{fox2003pddl2}.
This has the advantage that all the computational and meta-theoretic properties of $\mathcal{ES}$ apply to it.
On the other hand, it has the disadvantage of being less understandable than a state transition-based definition, as one needs to first understand $\mathcal{ES}$.
Seen from that perspective, our formalisation three properties:\begin{enumerate*}
\item It is clearly state transition-based as our semantics are in terms recursively defined action execution and state transitions.
This makes it more readable than the formalisation of~\citeauthor{DBLP:conf/aaai/ClassenHL07}.
\item It is also declarative in HOL since, although our top-level definitions are state transition-based, the mechanisms behind the recursive function definitions and the algebraic data types in HOL are all declarative in terms of the axioms of HOL~\cite{DBLP:phd/de/Krauss2009,DBLP:conf/lics/TraytelPB12}.
\item Has less clear computational properties, since general procedures to reason about HOL are all heuristic, since the logic is incomplete.
This disadvantage is not an issue, however, in our context given that our goal is to specify a concise semantics for deriving correct by construction software.
It can, nonetheless, be remedied by formalising the semantics of $\mathcal{ES}$ in HOL and formally showing, within Isabelle/HOL, the correctness of the encoding of PDDL in $\mathcal{ES}$ from~\citeauthor{DBLP:conf/aaai/ClassenHL07}.
\end{enumerate*}

A lot of work on trustworthiness in planning has focused on plan validation.
The state-of-the-art plan validator for temporal plans is VAL~\cite{howey2004val}.
Since VAL implements temporal planning semantics, which is rather involved, in C++, it is difficult to inspect VAL to make sure that it is free of bugs.
This, in a sense, defeats one of the main purposes of plan validators: they are supposed to boost trustworthiness by being much simpler than planning systems, making it less likely for them to have bugs and making them easier to inspect.
One motivation for our work was to avoid that problem by having a separate concise specification of the semantics which precisely describes what the validator implements.
These semantics are then formally connected to an efficient validator.
Another approach to temporal plan validation is the one by~\cite{DBLP:conf/aaai/CimattiMR17}, who compile a given planning problem and a candidate plan into a formula of temporal logic.
Plan validation then becomes a satisfiability task for an LTL formula.
From a trustworthiness perspective, this approach has the disadvantages that one has to trust the code that implements the compilation to LTL and, more importantly, either one has to trust an LTL model-checker or devise a validator that validates models of LTL formulae.
Our approach, on the other hand, trusts a much smaller code base, thanks to the LCF architecture of Isabelle/HOL.

\mohammad{If we do the validation of invariants, we should add to the review other work on domain validation. We can copy that from the 2017 paper.}

As future work, we would like to connect our formalisation of temporal planning to the formalisation of timed automata by~\citeauthor{DBLP:conf/tacas/WimmerM20}~\citeyear{DBLP:conf/tacas/WimmerM20}.
This would enable us to generate formally checkable certificates of unsolvability for temporal planning problems.
It would also enable formally verified checking of different properties of a planning domain similar to the ones by~\citeauthor{DBLP:conf/aaai/CimattiMR17}, but with formal guarantees.

\paragraph{Acknowledgements}
This work was facilitated through the DFG Koselleck Grant NI 491/16-1.

\bibliography{long_paper}
\clearpage
\pagebreak
\section{Appendix: Running Example}

The problem is specified in PDDL by the PDDL-domain \lstinline[style=isainline]{temp-elevators} and PDDL-problem \lstinline[style=isainline]{temp-elevators-prob1} in Listing \ref{lst:temp_elevs_dom} and \ref{lst:temp_elevs_prob}.

\begin{PddlListing}[label={lst:temp_elevs_dom}]{PDDL-domain for elevator planning problem}
(define (domain temp-elevators)
  (:requirements :typing :negative-preconditions :durative-actions :duration-inequalities)
  (:types |\floorType| - object |\elevatorType| - object |\passengerType| - object)|\label{line:temp_elevs_types}|
  (:predicates|\label{line:temp_elevs_preds}|
      (|\elevatorAtPred| ?e - |\elevatorType| ?f - |\floorType|)
      (|\passengerAtPred| ?e - |\passengerType| ?f - |\floorType|)
      (|\inElevatorPred| ?p - |\passengerType| ?e - |\elevatorType|)
      (|\elevatorDoorOpenPred| ?e - |\elevatorType|)
  )
  (:functions (|\elevatorDurationFunc| ?from - |\floorType| ?to - |\floorType|) - number)|\label{line:temp_elevs_funcs}|
  (:durative-action |\moveElevatorAct||\label{line:temp_elevs_moveelev}|
      :parameters (?e - |\elevatorType| ?from - |\floorType| ?to - |\floorType|)
      :duration (= ?duration (|\elevatorDurationFunc| ?from ?to))|\label{line:temp_elevs_moveelev_dconst}|
      :condition (and (at start (|\elevatorAtPred| ?e ?from))
                      (over all (not (|\elevatorDoorOpenPred| ?e))))|\label{line:temp_elevs_moveelev_cond}|
      :effect (and (at start (not (|\elevatorAtPred| ?e ?from))) 
                   (at end (|\elevatorAtPred| ?e ?to)))|\label{line:temp_elevs_moveelev_eff}|
  )
  (:durative-action |\openElevatorDoorAct||\label{line:temp_elevs_openelevdoor}|
      :parameters (?e - |\elevatorType|)
      :duration (= ?duration 1)
      :condition (at start (not (|\elevatorDoorOpenPred| ?e)))
      :effect (at end (|\elevatorDoorOpenPred| ?e))
  )
  (:durative-action |\closeElevatorDoorAct||\label{line:temp_elevs_closeelevdoor}|
      :parameters (?e - |\elevatorType|)
      :duration (= ?duration 1)
      :condition (at start (|\elevatorDoorOpenPred| ?e))
      :effect (at end (not (|\elevatorDoorOpenPred| ?e)))
  )
  (:durative-action |\enterElevatorAct||\label{line:temp_elevs_enterelev}|
      :parameters (?p - |\passengerType| ?e - |\elevatorType| ?f - |\floorType|)
      :duration (<= ?duration 1)
      :condition (and (at start (and (|\passengerAtPred| ?p ?f) (|\elevatorAtPred| ?e ?f)))
                      (over all (|\elevatorDoorOpenPred| ?e)))
      :effect (and (at start (not (|\passengerAtPred| ?p ?f))) 
                   (at end (|\inElevatorPred| ?p ?e)))
  )
  (:durative-action |\exitElevatorAct||\label{line:temp_elevs_exitelev}|
      :parameters (?p - |\passengerType| ?e - |\elevatorType| ?f - |\floorType|)
      :duration (<= ?duration 1)
      :condition (and (at start (and (|\inElevatorPred| ?p ?e) (|\elevatorAtPred| ?e ?f)))
                      (over all (|\elevatorDoorOpenPred| ?e)))
      :effect (and (at start (not (|\inElevatorPred| ?p ?e)))
                   (at end (|\passengerAtPred| ?p ?f)))
  )
)
\end{PddlListing}
\begin{PddlListing}[label={lst:temp_elevs_prob}]{PDDL-problem for elevator planning problem}
(define (problem temp-elevators-prob1)
(:domain elevators)
(:objects f0 f1 - |\floorType| p0 p1 - |\passengerType| e0 e1 - |\elevatorType|)
(:init (|\elevatorAtPred| e0 f0)
    (|\elevatorAtPred| e1 f1)
    (|\elevatorDoorOpenPred| e0)
    (|\passengerAtPred| p0 f1)
    (|\passengerAtPred| p1 f0)
    (= (|\elevatorDurationFunc| f0 f1) 1)
    (= (|\elevatorDurationFunc| f1 f0) 1))
(:goal (and (|\passengerAtPred| p0 f0) (|\passengerAtPred| p1 f1))))
\end{PddlListing}
In the running example, there are two passengers \lstinline[style=isainline]{p0} and \lstinline[style=isainline]{p1}, who want to use two elevators \lstinline[style=isainline]{e0} and \lstinline[style=isainline]{e1}. Passenger \lstinline[style=isainline]{p0} wants to move from floor \lstinline[style=isainline]{f1} to floor \lstinline[style=isainline]{f0}, whereas the passenger \lstinline[style=isainline]{p1} wants to move from floor \lstinline[style=isainline]{f0} to floor \lstinline[style=isainline]{f1}. The PDDL-domain (listing \ref{lst:temp_elevs_dom}) specifies actions to move an elevator ({\lstinlinemacro\moveElevatorAct}), to enter and exit an elevator ({\lstinlinemacro\enterElevatorAct} and {\lstinlinemacro\exitElevatorAct}), and to open and close and elevator door ({\lstinlinemacro\openElevatorDoorAct} and {\lstinlinemacro\closeElevatorDoorAct}). Each one of the actions has the expected preconditions and effects; e.g. {\lstinlinemacro\moveElevatorAct} requires the elevator door to be closed during the entire move actions. For a passenger to enter an elevator, {\lstinlinemacro\enterElevatorAct} requires the elevator door to be open.
A PDDL-domain defines action schemas, which are instantiated with arguments to obtain executable ground actions. For the instantiation of an action schema the conditions and effects of are partitioned by their annotation (\lstinline[style=isainline]{at start}, \lstinline[style=isainline]{at end}, and \lstinline[style=isainline]{over all}) to produce the snap actions $\durActionStart$, $\durActionEnd$, and the invariants $\durActionInv$. The snap action $\durActionStart$ contains all instantiated conditions and effects annotated with \lstinline[style=isainline]{at start}, whereas the snap action $\durActionEnd$ contains all instantiated conditions and effects annotated with \lstinline[style=isainline]{at end}, and $\durActionInv$ contains all instantiated invariants.

\begin{PddlListing}[label={lst:temp_elevs_plan}]{Valid plan for elevator planning problem}
0: (|\openElevatorDoorAct| e1)[1]
1.25: (|\enterElevatorAct| p0 e1 f1)[0.5]
2: (|\closeElevatorDoorAct| e1)[1]
3: (|\moveElevatorAct| e1 f1 f0)[1]
4: (|\openElevatorDoorAct| e1)[1]
5.25: (|\exitElevatorAct| p0 e1 f0)[0.5]
0.75: (|\enterElevatorAct| p1 e0 f0)[0.5]
1.5: (|\closeElevatorDoorAct| e0)[1]
2.5: (|\moveElevatorAct| e0 f0 f1)[1]
3.5: (|\openElevatorDoorAct| e0)[1]
4.75: (|\exitElevatorAct| p1 e0 f1)[0.5]
\end{PddlListing}
Listing \ref{lst:temp_elevs_plan} shows a valid plan for the running example.
A plan for a temporal planning problem is a schedule of actions to execute.
The plan specifies the starting time point and a duration for each action.
In PDDL the starting time points are denoted in front of a colon (\lstinline[style=isainline]{:}) at the beginning of each line.
The duration of an action is denoted in brackets (\lstinline[style=isainline]{[d]}, with duration $d$) at the end of a line.

In the elevator example the end snap action for the action instantiation of {\lstinlinemacro{\closeElevatorDoorAct\ e0}} and the start snap action for the action instantiation of {\lstinlinemacro{\openElevatorDoorAct\ e0}} are interfering ground actions. 
\begin{itemize}
  \item $\text{(\lstinlinemacro{\closeElevatorDoorAct\ e0})}_\textit{end}=\langle\top,\emptyset,\{\elevatorDoorOpenPredInline{e0}\}\rangle$
  \item $\text{(\lstinlinemacro{\openElevatorDoorAct\ e0})}_\textit{start}=\langle\neg\elevatorDoorOpenPredInline{e0},\emptyset,\emptyset\rangle$
\end{itemize}
The end snap action for {\lstinlinemacro{\closeElevatorDoorAct\ e0}} contains the negative effect {\lstinlinemacro{\elevatorDoorOpenPred\ e0}}, which is a precondition of the start snap action of {\lstinlinemacro{\openElevatorDoorAct\ e0}}. This violates the first condition for non-interference between ground actions (Definition \ref{def:act_non_interf}). Intuitively this means, there has to be a non-zero, but arbitrarily small, time interval where the elevator door is open before it can be closed again.

\section{Appendix: Isabelle/HOL Listings}
Isabelle's syntax is a variation of Standard ML combined with standard mathematical notation.
Function application is written infix, and functions can be Curried, i.e.\ function $f$ applied to arguments $x_1~\ldots~x_n$ is written as $f~x_1~\ldots~x_n$ instead of the standard notation $f(x_1,~\ldots~,x_n)$.

\subsection{Formalising the Semantics}
\mohammad{We exploit the fact that Isabelle/HOL has a recursion combinator to express the semantics more elegantly.}
  
The following code snippet shows our formalization of the ground action interference according to Definition~\ref{def:act_non_interf}.
\begin{IsabelleSnippet}[label=isa:groundacts]{Ground Action interference and Happening Sequence validity}
definition acts_non_intrf :: "ground_action \<Rightarrow> ground_action \<Rightarrow> bool" where
  "acts_non_intrf a b \<longleftrightarrow>
  (let
    add\<^sub>a = set(adds(effect a));
    del\<^sub>a = set(dels(effect a));
    pre\<^sub>a = Atom ` atoms (precondition a);
    add\<^sub>b = set(adds(effect b));
    del\<^sub>b = set(dels(effect b));
    pre\<^sub>b = Atom ` atoms (precondition b) in
    pre\<^sub>a \<inter> (add\<^sub>b \<union> del\<^sub>b) = {} \<and>
    pre\<^sub>b \<inter> (add\<^sub>a \<union> del\<^sub>a) = {} \<and>
    add\<^sub>a \<inter> del\<^sub>b = {} \<and>
    add\<^sub>b \<inter> del\<^sub>a = {})"
\end{IsabelleSnippet}
The formalization of the ground action interference is straight forward. The function \lstinline[style=isainline]{acts_non_intrf} returns a boolean value indicating whether the given ground actions are non-interfering according to Definition~\ref{def:act_non_interf}.

The following code snippet shows our formalization of Definition \ref{def:rec_sets}.
\begin{IsabelleSnippet}[label=isa:newsema]{New Semantics}
definition acts_of_plan_at :: "time \<Rightarrow> plan \<Rightarrow> ground_action set" where
  "acts_of_plan_at t\<^sub>i \<pi>s =
    {a\<^sub>\<pi>. \<exists>\<pi>. (t\<^sub>i,\<pi>) \<in> simple_acts \<pi>s \<and> Some a\<^sub>\<pi> = res_inst \<pi>}
  \<union> {a\<^sub>s\<^sub>t\<^sub>a\<^sub>r\<^sub>t. \<exists>\<pi>. (t\<^sub>i,\<pi>) \<in> durative_acts \<pi>s
  \<and> Some a\<^sub>s\<^sub>t\<^sub>a\<^sub>r\<^sub>t = res_inst_snap_action \<pi> At_Start}
  \<union> {a\<^sub>e\<^sub>n\<^sub>d. \<exists>t' \<pi>. (t',\<pi>) \<in> durative_acts \<pi>s
  \<and> t\<^sub>i = t' + duration \<pi>
  \<and> Some a\<^sub>e\<^sub>n\<^sub>d = res_inst_snap_action \<pi> At_End}"

definition invs_of_plan_at :: "time \<Rightarrow> plan \<Rightarrow> (object atom) formula set" where
  "invs_of_plan_at t\<^sub>i \<pi>s =
    {inv. \<exists>t\<^sub>\<pi> \<pi>. (t\<^sub>\<pi>,\<pi>) \<in> durative_acts \<pi>s
  \<and> t\<^sub>\<pi> < t\<^sub>i \<and> t\<^sub>i \<le> t\<^sub>\<pi> + (duration \<pi>)
  \<and> Some inv = res_inst_inv \<pi>}"

fun apply_eff :: "ground_action set \<Rightarrow> world_model \<Rightarrow> world_model" where
  "apply_eff A\<^sub>i M =
    (M - \<Union> (set ` dels ` effect ` A\<^sub>i)) \<union> \<Union> (set ` adds ` effect ` A\<^sub>i)"

fun valid_state_seq :: "world_model \<Rightarrow> time list \<Rightarrow> plan \<Rightarrow> world_model \<Rightarrow> bool" where
  "valid_state_seq M [] \<pi>s M' \<longleftrightarrow> (M = M')"
| "valid_state_seq M (t\<^sub>i#ts) \<pi>s M' \<longleftrightarrow>
  (let A\<^sub>i = acts_of_plan_at t\<^sub>i \<pi>s in
  (\<forall>i \<in> invs_of_plan_at t\<^sub>i \<pi>s. M   \<^sup>c\<TTurnstile>\<^sub>=   i)
  \<and> (\<forall>a \<in> A\<^sub>i. M   \<^sup>c\<TTurnstile>\<^sub>=   precondition a)
  \<and> (\<forall>a \<in> A\<^sub>i. \<forall>b \<in> A\<^sub>i. a \<noteq> b \<longrightarrow> acts_non_intrf a b)
  \<and> valid_state_seq (apply_eff A\<^sub>i M) ts \<pi>s M')"
\end{IsabelleSnippet}
\lukas{renamed \lstinline[style=isainline]{plan_action_path} -> \lstinline[style=isainline]{valid_state_seq} to make more sense with the defintions.}
Note that in our formalization we make a distinction between durative and non-durative (instantaneous) actions. Therefore, the function \lstinline[style=isainline]{simple_acts} return all non-durative plan actions in a given plan, and the function \lstinline[style=isainline]{durative_acts} return all durative plan actions in a given plan.
The function \lstinline[style=isainline]{res_inst_snap_action} instantiates the snap action of a given durative plan action for a given temporal annotation: \lstinline[style=isainline]{At_Start} or \lstinline[style=isainline]{At_End}, corresponding to $\durActionStart$ and $\durActionEnd$. The function \lstinline[style=isainline]{res_inst} instantiates the ground action for a non-durative plan action. The function \lstinline[style=isainline]{res_inst_inv} instantiates the invariant formula ($\durActionInv$) for a given durative plan action.
Given a time point \lstinline[style=isainline]{t} and a plan \lstinline[style=isainline]{\<pi>s}, the function call \lstinline[style=isainline]{acts_of_plan_at t \<pi>s} returns the set $\recSetActs_\happTimePoint$ from Definition \ref{def:rec_sets}. Similarly, given a time point \lstinline[style=isainline]{t} and a plan \lstinline[style=isainline]{\<pi>s}, the function call \lstinline[style=isainline]{invs_of_plan_at t \<pi>s} returns the set $\recSetInvs_\happTimePoint$ from Definition \ref{def:rec_sets}.

The function \lstinline[style=isainline]{apply_eff} applies the effects of a set of ground actions to a given state and returns the resulting state. The function \lstinline[style=isainline]{valid_state_seq} recursively formalizes the validity of a state sequence according to Definition \ref{def:rec_sets}. The state sequence is never explicitly constructed. In each recursion step it is checked that
\begin{enumerate*}
  \item all invariants for the current time point are satisfied by the current state,
  \item the precondition of each action for the current time point is satisfied by the current state, and
  \item all actions for the current time point are pairwise non-interfering.
\end{enumerate*}

The following code snippet then shows our formalization of plan validity according to Definition \ref{def:plan_validity_rec}.
\begin{IsabelleSnippet}[label=isa:newsemathm]{New Semantics}
lemma "valid_plan \<pi>s \<equiv> wf_plan \<pi>s
  \<and> (\<exists>htps M'. htps_seq \<pi>s htps
  \<and> valid_state_seq I htps \<pi>s M' \<and> M'   \<^sup>c\<TTurnstile>\<^sub>=   (goal P))"
  unfolding valid_plan_def valid_plan_from_def by auto
\end{IsabelleSnippet}
\lukas{renamed \lstinline[style=isainline]{valid_plan2} -> \lstinline[style=isainline]{valid_plan} to make more sense with the defintions.}
The predicate \lstinline[style=isainline]{wf_plan} characterizes a well-formed plan, and the predicate \lstinline[style=isainline]{htps_seq} characterizes the sequence of happening time points for a plan. The predicate \lstinline[style=isainline]{valid_plan} characterizes the plan validity of a plan according to Definition \ref{def:plan_validity_rec}.

\subsection{Formalising the Refined Semantics}

Next we are showing our formalization of our refined semantics, which are based around the induced happening sequence (Definition \ref{def:ind_happ_seq} and \ref{def:val_happ_seq}).

The following code snippet shows our formalization of the happening execution (\lstinline[style=isainline]{apply_happ}) and the validity of a happening sequence (\lstinline[style=isainline]{valid_happ_seq}).
\begin{IsabelleSnippet}[label=isa:indhapseqa]{Induced Happening Sequence and Plan validity}
fun apply_happ :: "happening \<Rightarrow> world_model \<Rightarrow> world_model" where
  "apply_happ (t\<^sub>i, A\<^sub>i) M =
  (M - \<Union> (set (map (set o dels o effect) A\<^sub>i))) \<union> \<Union> (set (map (set o adds o effect) A\<^sub>i))"

fun valid_happ_seq :: "world_model \<Rightarrow> happening list \<Rightarrow> world_model \<Rightarrow> bool" where
  "valid_happ_seq M [] M' \<longleftrightarrow> (M = M')"
| "valid_happ_seq M ((t\<^sub>i,A\<^sub>i)#hs) M' \<longleftrightarrow>
  (\<forall>a\<in> set A\<^sub>i. M   \<^sup>c\<TTurnstile>\<^sub>=   precondition a)
  \<and> (\<forall>a \<in> set A\<^sub>i. \<forall>b \<in> set A\<^sub>i. a \<noteq> b \<longrightarrow> acts_non_intrf a b)
  \<and> valid_happ_seq (apply_happ (t\<^sub>i,A\<^sub>i) M) hs M'"
\end{IsabelleSnippet}
The function \lstinline[style=isainline]{apply_happ} returns the state after applying the effects of a given happening to a given state. The validity of a happening sequence (Definition~\ref{def:val_happ_seq}) is formalized with the function \lstinline[style=isainline]{valid_happ_seq}, which is recursive on the given happening sequence. The state sequence is only constructed implicitly through the recursion.
In each recursion step it is checked that
\begin{enumerate*}
  \item the actions within the happening are pairwise non-interfering and
  \item the preconditions of each action in the happening are satisfied by the current state.
\end{enumerate*}

The following code snippet shows our formalization of the predicate that characterizes induced happening sequences (Definition \ref{def:ind_happ_seq}).
\begin{IsabelleSnippet}[label=isa:indhapseqb]{Induced Happening Sequence}
definition ind_happ_seq :: "plan \<Rightarrow> happening list \<Rightarrow> bool" where
  "ind_happ_seq \<pi>s hs \<longleftrightarrow>
  (strict_sorted (map fst hs)
  \<and> (\<forall>(t\<^sub>\<pi>,\<pi>) \<in> simple_acts \<pi>s.
    let g\<^sub>a = the (res_inst \<pi>) in
    \<exists>A. (t\<^sub>\<pi>,A) \<in> set hs \<and> g\<^sub>a \<in> set A)
  \<and> (\<forall>(t\<^sub>\<pi>,\<pi>) \<in> durative_acts \<pi>s.
    let \<pi>\<^sub>s\<^sub>t\<^sub>a\<^sub>r\<^sub>t = the (res_inst_snap_action \<pi> At_Start);
      \<pi>\<^sub>e\<^sub>n\<^sub>d = the (res_inst_snap_action \<pi> At_End);
      \<pi>\<^sub>i\<^sub>n\<^sub>v = the (res_inst_snap_action \<pi> Over_All) in
    (\<exists>A. (t\<^sub>\<pi>,A) \<in> set hs \<and> \<pi>\<^sub>s\<^sub>t\<^sub>a\<^sub>r\<^sub>t \<in> set A) \<and>
    (\<exists>A. (t\<^sub>\<pi> + (duration \<pi>),A) \<in> set hs \<and> \<pi>\<^sub>e\<^sub>n\<^sub>d \<in> set A) \<and>
    (\<forall>t\<^sub>i t\<^sub>j. (consec_htps \<pi>s t\<^sub>i t\<^sub>j
  \<and> t\<^sub>\<pi> \<le> t\<^sub>i \<and> t\<^sub>j \<le> t\<^sub>\<pi> + (duration \<pi>))
    \<longrightarrow> (\<exists>(t',A) \<in> set hs.
      t\<^sub>i < t' \<and> t' < t\<^sub>j \<and> \<pi>\<^sub>i\<^sub>n\<^sub>v \<in> set A)))
  \<and> (\<forall>(t\<^sub>i,A\<^sub>i) \<in> set hs. A\<^sub>i \<noteq> [] \<and>
    (\<forall>a \<in> set A\<^sub>i. \<exists>(t\<^sub>\<pi>,\<pi>) \<in> set \<pi>s.
      inst_of_plan_action \<pi>s (t\<^sub>\<pi>,\<pi>) (t\<^sub>i,a))))"
\end{IsabelleSnippet}
An induced happening sequence is characterized by the following constraints:
\begin{enumerate*}
  \item the induced happening sequence is strictly sorted by the time points of each happening,
  \item for each plan action the induced happening sequence contains the correct snap actions (according to Definition \ref{def:ind_happ_seq}), and
  \item every action in the induced happening sequence is a grounded and instantiated snap action for a plan action.
\end{enumerate*}
The predicate \lstinline[style=isainline]{ind_happ_seq} formalizes these constraints and consists of four conjuncts. The first conjunct ensures that an induced happening sequence is strictly sorted by the time points of happenings. The second and third conjuncts specify the placement of the instantiated ground actions for each plan action according to Definition \ref{def:ind_happ_seq}.
The fourth and final conjunct uses the function \lstinline[style=isainline]{inst_of_plan_action} to ensure that all ground actions in the induced happening sequence are a instantiated snap actions for a plan action.

The following code snippet then shows our formalization of plan validity according to Definition \ref{def:val_plan}.
\begin{IsabelleSnippet}[label=isa:indhapseqc]{Plan validity (II)}
definition plan_happ_path :: "world_model \<Rightarrow> plan \<Rightarrow> world_model \<Rightarrow> bool" where
  "plan_happ_path M \<pi>s M' \<longleftrightarrow>
  (\<exists>hs. ind_happ_seq \<pi>s hs \<and> valid_happ_seq M hs M')"

lemma "valid_plan2 \<pi>s \<equiv> wf_plan \<pi>s
  \<and> (\<exists>M'. plan_happ_path I \<pi>s M' \<and> M'   \<^sup>c\<TTurnstile>\<^sub>=   goal P)"
  unfolding valid_plan2_def valid_plan_from2_def by auto
\end{IsabelleSnippet}
The formalization of the plan validity \lstinline[style=isainline]{valid_plan2} uses an existential quantifier for a valid induced happening sequence.

The next code snippet shows the Lemma that proves the equivalence between the two plan validity definitions (Definition \ref{def:plan_validity_rec} and \ref{def:val_plan}).
\begin{IsabelleSnippet}[label=isa:indhapseqd]{Equivalence Proof}
lemma
  assumes "wf_problem"
  shows "valid_plan \<pi>s \<longleftrightarrow> valid_plan2 \<pi>s"
  unfolding valid_plan_def valid_plan2_def
  using valid_plan_from2_iff[OF assms] by blast
\end{IsabelleSnippet}
Under the assumptions that the given planning problem and plan are well-formed (\lstinline[style=isainline]{wf_problem}) both formalizations (\lstinline[style=isainline]{valid_plan} and \lstinline[style=isainline]{valid_plan2}) are equivalent.

\lukas{notes on difficulty of formal proof. compare sizes -> in appendix}

\subsection{Formalization of Executable Plan Validator}

First, executable refinements for the semantics of ground actions and happenings are implemented.
\begin{IsabelleSnippet}[label=isa:sattodimacs]{Enabled-ness Execution of a Happening}
definition en_exE :: "happening \<Rightarrow> world_model \<Rightarrow> _+world_model" where
  "en_exE \<equiv> \<lambda>(t\<^sub>i,A\<^sub>i) \<Rightarrow> \<lambda>s. do {
  check_allm (\<lambda>a. check (holds s (precondition a))
    (ERRS ''Precondition not satisfied'')) A\<^sub>i;
  check_pairwise (\<lambda>a b.
    check (a \<noteq> b \<longrightarrow> acts_non_intrf a b)
      (ERRS ''Actions in happening interfering'')) A\<^sub>i;
  Error_Monad.return (apply_happ_exec (t\<^sub>i,A\<^sub>i) s)}"
\end{IsabelleSnippet}
The function \lstinline[style=isainline]{en_exE} combines the execution of a happening with an enabled-ness check. A happening $h$ is \textit{enabled} in a state if
\begin{enumerate*}[label=(\roman*)]
  \item the precondition of every ground action in the happening $h$ is satisfied in the state $M$, and
  \item no two ground actions in the happening $h$ are interfering.
\end{enumerate*}
If a given happening $h$ is enabled in a given state $M$ then \lstinline[style=isainline]{en_exE} return the resulting state after applying the effects of the happening $h$. The function \lstinline[style=isainline]{apply_happ_exec} is an executable refinement of the function \lstinline[style=isainline]{apply_happ}.

The following lemma justifies the refinement for the happening execution.
\begin{IsabelleSnippet}[label=isa:sattodimacs]{Justification Enabled-ness Execution of a Happening}
lemma (in wf_ast_problem) en_exE_return_iff:
  assumes "wm_basic s"
  and "\<forall>a \<in> set A\<^sub>i. wf_ground_action a"
  shows "en_exE (t\<^sub>i,A\<^sub>i) s = Inr s' \<longleftrightarrow> happ_enabled (t\<^sub>i,A\<^sub>i) s \<and> s' = apply_happ (t\<^sub>i,A\<^sub>i) s"
  unfolding en_exE_def
  using assms holds_for_wf_fmlas[OF \<open>wm_basic s\<close>]
  symmetric_pred_check_pairwise[OF acts_non_intrf_symmetric]
  apply_happ_exec_refine
  by auto
\end{IsabelleSnippet}
The assumptions for this lemma are that the world model is \textit{basic} and all ground actions in the happening are well-formed. A \textit{basic} world model only contains predicate atoms. The lemma then states, that the function \lstinline[style=isainline]{en_exE} only returns \lstinline[style=isainline]{Inr s'} if the given happening \lstinline[style=isainline]{(t\<^sub>i,A\<^sub>i)} is enabled in the state \lstinline[style=isainline]{s} and the application of the effects yield the world model \lstinline[style=isainline]{s'}. Otherwise the function will return \lstinline[style=isainline]{Inl msg} with an error message \lstinline[style=isainline]{msg}.

The validity of a happening sequence (Definition \ref{def:val_happ_seq}) is implemented recursively and combined with the enabled-ness execution for happenings. Moreover, the entailment of the goal state specifications is directly checked in the base case of the recursion.
\begin{IsabelleSnippet}[label=isa:sattodimacs]{Validity of a Happening Sequence}
fun valid_happ_seq_fromE :: "nat \<Rightarrow> world_model \<Rightarrow> happening list \<Rightarrow> _+unit" where
  "valid_happ_seq_fromE si s [] =
  check (holds s (goal P))
    (ERRS ''Postcondition does not hold'')"
| "valid_happ_seq_fromE si s (h#hs) = do {
  s \<leftarrow> en_exE h s <+? (\<lambda>e _. shows ''at step '' o shows si o shows '': '' o e ());
  valid_happ_seq_fromE (si+1) s hs}"
\end{IsabelleSnippet}
The argument \lstinline[style=isainline]{si} is the index of the current execution step. The implementation is justified with the following lemma.
\begin{IsabelleSnippet}[label=isa:sattodimacs]{Justification for Validity of a Happening Sequence}
lemma (in wf_ast_problem)
  assumes "wm_basic M" and "wf_happ_seq hs"
  shows "valid_happ_seq_fromE k M hs = Inr () \<longleftrightarrow>
  (\<exists> M'. valid_happ_seq M hs M' \<and> M'   \<^sup>c\<TTurnstile>\<^sub>=   (goal P))"
  using assms valid_happ_seq_from_refine valid_happ_seq_fromE_return_iff by auto
\end{IsabelleSnippet}
The function \lstinline[style=isainline]{valid_happ_seq_fromE} only returns \lstinline[style=isainline]{Inr ()} if there exists a state \lstinline[style=isainline]{M'}, such that the given happening sequence \lstinline[style=isainline]{hs} is valid from the starting state \lstinline[style=isainline]{M} to the state \lstinline[style=isainline]{M'} and \lstinline[style=isainline]{M'} entails the goal state specifications.

Next the executable refinements for the semantics of temporal plan validity are described.

The following code snippet shows the implementation of the function \lstinline[style=isainline]{simplify_plan}, which is an executable function that produces an induced happening sequence.
\begin{IsabelleSnippet}[label=isa:sattodimacs]{Executable Construction of an Induced Happening Sequence}
fun insort_happ :: "happening \<Rightarrow> happening list \<Rightarrow> happening list" where
  "insort_happ (t\<^sub>i,A\<^sub>i) [] = [(t\<^sub>i,A\<^sub>i)]"
| "insort_happ (t\<^sub>i,A\<^sub>i) ((t\<^sub>j,A\<^sub>j)#hs) =
  (if t\<^sub>i < t\<^sub>j then (t\<^sub>i,A\<^sub>i)#(t\<^sub>j,A\<^sub>j)#hs
  else if t\<^sub>i = t\<^sub>j then (t\<^sub>j,A\<^sub>i @ A\<^sub>j)#hs
  else (t\<^sub>j,A\<^sub>j)#(insort_happ (t\<^sub>i,A\<^sub>i) hs))"

fun insort_mult_happs :: "happening list \<Rightarrow> happening list \<Rightarrow> happening list" where
  "insort_mult_happs [] hs\<^sub>2 = hs\<^sub>2"
| "insort_mult_happs (h#hs\<^sub>1) hs\<^sub>2 = insort_happ h (insort_mult_happs hs\<^sub>1 hs\<^sub>2)"

definition consec_htps_in_interval :: "time list \<Rightarrow> time \<Rightarrow> time \<Rightarrow> (time \<times> time) list" where
  "consec_htps_in_interval htps t\<^sub>i t\<^sub>j =
  filter (\<lambda>(t,t'). t\<^sub>i \<le> t \<and> t' \<le> t\<^sub>j) (zip (butlast htps) (tl htps))"

fun simplify_action :: "time list \<Rightarrow> (time \<times> plan_action) \<Rightarrow> (time \<times> ground_action) list" where
  "simplify_action htps (t,Simple_Plan_Action n args) = (
  let a = the (resolve_action_schema n) in
    [(t, instantiate_action_schema a args)])"
| "simplify_action htps (t,Durative_Plan_Action n args d) = (
  let a = the (resolve_action_schema n) in
  (t,inst_snap_action a args At_Start) #
  (t+d,inst_snap_action a args At_End) #
  (map (\<lambda>(t\<^sub>i,t\<^sub>j).((t\<^sub>i+t\<^sub>j) / 2,inst_snap_action a args Over_All)) (consec_htps_in_interval htps t (t+d))))"

fun simplify_plan :: "time list \<Rightarrow> plan \<Rightarrow> happening list" where
  "simplify_plan htps [] = []"
| "simplify_plan htps (\<pi>#\<pi>s) =
  insort_mult_happs (map (\<lambda>(t\<^sub>a,a). (t\<^sub>a,[a])) (simplify_action htps \<pi>)) (simplify_plan htps \<pi>s)"
\end{IsabelleSnippet}
The function \lstinline[style=isainline]{simplify_plan} takes two arguments: a sequence of happening time points and a plan. The sequence of happening time points is needed to place the invariant snap actions. The invariant snap actions are placed at $\frac{t_i+t_{i+1}}{2}$ for two consecutive happening time points $t_i$ and $t_{i+1}$. The function \lstinline[style=isainline]{simplify_plan} is recursive on the given plan. In each recursion step alls snap actions for the current plan action (according to Definition \ref{def:ind_happ_seq}) are inserted into the existing happening sequence with the function \lstinline[style=isainline]{insort_mult_happs}. The function \lstinline[style=isainline]{insort_mult_happs} uses the function \lstinline[style=isainline]{insort_happ} to inserts multiple happenings into a existing happening sequence. The function \lstinline[style=isainline]{simplify_action} return a list of all snap actions for a given plan action (according to Definition \ref{def:ind_happ_seq}). The function \lstinline[style=isainline]{consec_htps_in_interval} returns all intervals of consecutive happening time points that lie in between two given time points.

The following lemma proves the correctness of the function \lstinline[style=isainline]{simplify_plan}.
\begin{IsabelleSnippet}[label=isa:sattodimacs]{Correctnes of the Construction of an Induced Happening Sequence}
lemma (in wf_ast_problem)
  assumes "wf_plan \<pi>s"
  shows "ind_happ_seq \<pi>s (simplify_plan (htps_exec \<pi>s) \<pi>s)"
  using assms by (rule simplify_plan_correct)
\end{IsabelleSnippet}
The lemma proves, that given a well-formed plan (\lstinline[style=isainline]{wf_plan \<pi>s}) the function \lstinline[style=isainline]{simplify_plan} constructs an induced happening sequence for the given plan \lstinline[style=isainline]{\<pi>s}. The function \lstinline[style=isainline]{htps_exec} constructs the sequence of happening time points for a given plan.

\subsection{Efficiency Refinement of Executable Plan Validator}

The following code snippet shows the more efficient refinement of the function \lstinline[style=isainline]{simplify_plan}, that uses an AVL-tree instead of a list.
\begin{IsabelleSnippet}[label=isa:sattodimacs]{Efficient Construction of the Induced Happening Sequence with an AVL-tree}
fun insert_happ_to_tree :: "happening \<Rightarrow> happening tree_ht \<Rightarrow> happening tree_ht" where
  "insert_happ_to_tree (t\<^sub>i,A\<^sub>i) \<langle>\<rangle> = avl_node \<langle>\<rangle> (t\<^sub>i,A\<^sub>i) \<langle>\<rangle>"
| "insert_happ_to_tree (t\<^sub>i,A\<^sub>i) \<langle>l,((t\<^sub>j,A\<^sub>j),h),r\<rangle> = (
  if t\<^sub>i < t\<^sub>j then avl_balL (insert_happ_to_tree (t\<^sub>i,A\<^sub>i) l) (t\<^sub>j,A\<^sub>j) r
  else if t\<^sub>i = t\<^sub>j then avl_node l (t\<^sub>j,A\<^sub>i @ A\<^sub>j) r
  else avl_balR l (t\<^sub>j,A\<^sub>j) (insert_happ_to_tree (t\<^sub>i,A\<^sub>i) r))"
  
fun insert_timed_ground_acts_to_tree :: "(time \<times> ground_action) list \<Rightarrow> happening tree_ht \<Rightarrow> happening tree_ht" where
  "insert_timed_ground_acts_to_tree [] htree = htree"
| "insert_timed_ground_acts_to_tree ((t\<^sub>a,a)#as) htree =
  insert_happ_to_tree (t\<^sub>a,[a]) (insert_timed_ground_acts_to_tree as htree)"
  
fun simplify_planE_avl' :: "_ \<Rightarrow> (object, type) mapping \<Rightarrow> (object, rat) mapping \<Rightarrow> time list \<Rightarrow> plan \<Rightarrow> _+happening tree_ht" where
  "simplify_planE_avl' stg mp mp_fe htps [] = do { Error_Monad.return \<langle>\<rangle> }"
| "simplify_planE_avl' stg mp mp_fe htps (\<pi>#\<pi>s) = do {
  as \<leftarrow> simplify_actionE G mp mp_fe htps \<pi>;
  htree \<leftarrow> simplify_planE_avl' G mp mp_fe htps \<pi>s;
  Error_Monad.return (insert_timed_ground_acts_to_tree as htree)}"

fun simplify_planE_avl :: "_ \<Rightarrow> (object, type) mapping \<Rightarrow> (object, rat) mapping \<Rightarrow> time list \<Rightarrow> plan \<Rightarrow> _+happening list" where
  "simplify_planE_avl stg mp mp_fe htps \<pi>s = do {
  htree \<leftarrow> simplify_planE_avl' G mp mp_fe htps \<pi>s;
  Error_Monad.return (avl_inorder htree)}"
\end{IsabelleSnippet}
The function \lstinline[style=isainline]{simplify_planE_avl'} constructs an AVL-tree containing all happenings for the constructed induced happening sequence. The function \lstinline[style=isainline]{simplify_planE_avl} then simply uses an inorder-traversal (\lstinline[style=isainline]{avl_inorder}) on the AVL-tree to obtain the induced happening sequence.

The argument \lstinline[style=isainline]{stg} is the instantiated subtype relation for an implicitly fixed domain and the argument \lstinline[style=isainline]{mp} is a map from object names to types. For a given plan action the function \lstinline[style=isainline]{simplify_actionE} returns all snap actions and checks the well-formedness of the given plan action. The implementations for the efficient type checking and its verification are reused from \cite{abdulaziz2018formallyVal} without any modifications. The functions \lstinline[style=isainline]{insert_timed_ground_acts_to_tree} and \lstinline[style=isainline]{insert_happ_to_tree} are used to insert elements to the AVL tree.

\begin{IsabelleSnippet}[label=isa:sattodimacs]{Correctnes of the Construction of the Induced Happening Sequence}
lemma (in wf_ast_problem)
  assumes "wf_domain"
  and "simplify_planE_avl STG mp_objT mp_Fevl (htps_exec \<pi>s) \<pi>s = Inr hs"
  shows "ind_happ_seq \<pi>s hs" and "wf_plan \<pi>s"
  using assms simplify_planE_avl_equiv simplify_planE_return_iff
  htps_exec_sorted htps_exec_distinct simplify_plan_correct
  by (auto simp: strict_sorted_iff)
\end{IsabelleSnippet}
This lemma proves that if \lstinline[style=isainline]{simplify_planE_avl} returns \lstinline[style=isainline]{Inr hs} then the returned happening sequence \lstinline[style=isainline]{hs} is an induced happening sequence and the plan \lstinline[style=isainline]{\<pi>s} is well-formed.

Finally, the construction of the induced happening sequence (\lstinline[style=isainline]{simplify_planE_avl}) and the implementation for the validity a happening sequence (\lstinline[style=isainline]{valid_happ_seq_fromE}) are combined in the function \lstinline[style=isainline]{check_plan}, which gives us the desired validator.
\begin{IsabelleSnippet}[label=isa:sattodimacs]{Implementation of Validtor function}
definition "check_plan P \<pi>s \<equiv> do {
  let stg = ast_domain.STG (ast_problem.domain P);
  let conT = ast_domain.mp_constT (ast_problem.domain P);
  let mp = ast_problem.mp_objT P;
  let mp_fe = ast_problem.mp_Fevl P;
  check_wf_problem P stg conT mp;
  hs \<leftarrow> ast_problem.simplify_planE_avl stg mp mp_fe (htps_exec \<pi>s) \<pi>s;
  ast_problem.valid_happ_seq_fromE 1 (ast_problem.I P) hs
} <+? (\<lambda>e. String.implode (e () ''''))"
\end{IsabelleSnippet}

The following theorem proves our validator correct.
\begin{IsabelleSnippet}[label=isa:sattodimacs]{Correctness of our Validator}
theorem "check_plan P \<pi>s = Inr () \<longleftrightarrow> ast_problem.wf_problem P \<and> ast_problem.valid_plan P \<pi>s"
  by (rule check_plan_return_iff)
\end{IsabelleSnippet}
Our validator \lstinline[style=isainline]{check_plan} only returns \lstinline[style=isainline]{Inr ()} if the problem is well-formed and the plan is valid.

 \end{document}